\documentclass{article} 
\usepackage{iclr2024_conference,times}
\usepackage{amsmath}
\usepackage{amssymb}
\usepackage{xcolor}
\usepackage{amsfonts}
\usepackage{amsthm}
\usepackage{appendix}
\usepackage{hyperref}
\usepackage{url}
\usepackage{minitoc}

\DeclareSymbolFont{bbold}{U}{bbold}{m}{n}
\DeclareSymbolFontAlphabet{\mathbbold}{bbold}

\newcommand{\mycomment}[1]{}
\newcommand{\f}{\mathcal{F}}

\newcommand{\nn}{\mathcal{NN}}

\newtheorem{prop}{Proposition}
\newtheorem{lemma}{Lemma}

\newtheorem{assumption}{Assumption}

\usepackage{amsmath,amsfonts,bm}

\def\eqref#1{equation~\ref{#1}}

\def\1{\bm{1}}

\DeclareMathAlphabet{\mathsfit}{\encodingdefault}{\sfdefault}{m}{sl}
\SetMathAlphabet{\mathsfit}{bold}{\encodingdefault}{\sfdefault}{bx}{n}

\newcommand{\E}{\mathbb{E}}

\newcommand{\R}{\mathbb{R}}

\usepackage{hyperref}
\usepackage{url}

\title{Score-based generative models break the curse of dimensionality in learning a family of sub-Gaussian probability distributions}

\author{Frank Cole \\ 
Department of Mathematics\\
University of Minnesota\\
Minneapolis, MN, 55414, USA \\
\texttt{\{cole0932\}@umn.edu} \\
\And
Yulong Lu \\
Department of Mathematics\\
University of Minnesota\\
Minneapolis, MN, 55414, USA \\
\texttt{\{yulonglu\}@umn.edu} \\
\AND
}

\iclrfinalcopy 
\begin{document}
\doparttoc 
\faketableofcontents 

\part{} 

\maketitle

\begin{abstract}
While score-based generative models (SGMs) have achieved remarkable successes in enormous image generation tasks, their mathematical foundations are still limited. In this paper, we analyze the approximation and generalization of SGMs in learning a family of sub-Gaussian probability distributions. We introduce a notion of complexity for probability distributions in terms of their relative density with respect to the standard Gaussian measure. We prove that if the log-relative density can be locally approximated by a neural network whose parameters can be suitably bounded, then the distribution generated by empirical score matching approximates the target distribution in total variation with a dimension-independent rate. We illustrate our theory through examples, which include certain mixtures of Gaussians. An essential ingredient of our proof is to derive a dimension-free deep neural network approximation rate for the true score function associated to the forward process, which is interesting in its own right. 
\end{abstract}

\section{Introduction}
Generative modeling is a central task in modern machine learning, where the goal is to learn a high dimensional probability distribution given a finite number of samples. Score-based generative models (SGMs) \citet{sohl2015deep, song2020score}) recently arise as a novel family of generative models achieving remarkable empirical success in the generation of audio and images \citet{yang2022diffusion,croitoru2023diffusion}, even outperforming state-of-the-art generative models such as generative adversarial networks \citet{dhariwal2021diffusion}. More recently, SGMs have proven effective in a variety of applications such as natural language processing \citet{austin2021structured, savinov2021step}, computational physics \citet{lee2023score, jing2022torsional}, computer vision \citet{amit2021segdiff, baranchuk2021label, brempong2022denoising}, and medical imaging \citet{chung2022score}. In addition to their own expressive power, SGMs can also help to understand and improve other existing generative models, such as variational autoencoders \citet{huang2021variational, luo2022understanding} and normalizing flows \citet{gong2021interpreting}). \mycomment{A large amount of recent theoretical research has been devoted to demystifying the experimental triumphs of SGMs.}

SGMs are often implemented by a pair of diffusion processes, known as forward and backward processes. The forward process transforms given data into pure Gaussian noise, while the backward process turns the noises into approximate samples from the target distribution, thereby accomplishing generative modeling. The analytical form of the reverse process is unknown, since its parameters depend on the target distribution, which is only accessible through data; hence, the reverse process must be \textit{learned}. This is made possible by the remarkable fact that the time reversal of an diffusion process is again a diffusion process whose coefficients depend only on the target distribution via the \textit{score function}, a time-dependent vector field given by the gradient of the log-density of the forward process. There exist well-studied techniques to cast the estimation of the score function from data as a supervised learning problem \citet{hyvarinen2005estimation, vincent2011connection}, which is crucial to the practical implementation of SGMs.

While SGMs have received significant attention from theoretical viewpoints, there are still several barriers to a complete theoretical understanding. Recent results \citet{chen2023improved, chen2022sampling, benton2023linear} have shown that the distribution recovery error of SGMs is essentially controlled by the estimation error of the score function, which is typically parameterized by a neural network. While neural networks are known to be universal approximators for many classes of functions \citet{cybenko1989approximation, yarotsky2017error}, the number of parameters of the neural network needed to approximate a function to error $\epsilon$ often scales like $\epsilon^{-d}$, where $d$ is the dimension of the data. Such rates are of little practical significance for high dimensional problems, and thus the ability of neural networks to express the score function of a general probability distribution remains a mystery.

Nonetheless, SGMs have still exhibited great success in generating high-quality samples from complex, high-dimensional data distributions. One salient reason for this is that, while the data itself may be very high-dimensional, the score function of the noising process often possesses some intrinsic structure that can be exploited by neural networks. The purpose of this article is to justify this intuition rigorously for a broad class of probability distributions. Specifically, we study the generative power of SGMs for probability distributions which are absolutely continuous with respect to the standard Gaussian distribution. Such distributions admit a probability density function of the form
\begin{equation}\label{barrondata}
    p(x) = \frac{1}{Z} \exp\left(-\frac{\|x\|^2}{2} + f(x) \right),
\end{equation}
where $\exp(f): \R^d \rightarrow \R^+$ is the Radon-Nikodym derivative of $p$ with respect to the Gaussian distribution and $Z$ is the normalization constant. This representation of the density has proven particularly elucidating in the context of statistical learning and Bayesian inference, where the Gaussian component can model our subjective beliefs on the data. In this paper, we show that the expression for the density in Equation \ref{barrondata} is also relevant to SGMs, because the score function is related to the function $f$ by a tractable composition of functions. A central theme of this work is that if $f$ belongs to a low-complexity function class, then the score function inherits a similar low-complexity structure. This enables deep neural networks to learn the score function of diffusion processes without the curse of dimensionality in some concrete cases.

\subsection{Our contributions}
We summarize our contributions as follows.
\begin{enumerate}
    \item We prove that if the log-relative density of the data distribution with respect to the standard Gaussian can be locally approximated without the curse of dimensionality, then the score function at any fixed time $t$ can be approximated in the $L^2(p_t)$ norm, where $p_t$ denotes the marginal density of the forward process at time $t$, without the curse of dimensionality.
    \item We show that the empirical score matching estimator within a prescribed class neural networks can estimate the score at any fixed time without the curse of dimensionality. The error is decomposed into the approximation error of the score and the Rademacher complexity of the neural network class.
    \item We combine our results with existing discretization error bounds (e.g., in \citet{chen2022sampling}) to obtain explicit error estimates for SGMs in terms of the number of training samples. As an application, we prove that SGMs can sample from certain Gaussian mixture distributions with dimension-independent sample complexity.
\end{enumerate}

\subsection{Related work}
A majority of the recent theoretical analysis of SGMs \citet{de2021diffusion,lee2022convergence,lee2023convergence, chen2023improved,chen2023probability,benton2023linear} focuses on obtaining convergence guarantees for SGMs under minimal assumptions on the target distribution but, crucially, under the assumption that the score estimator is accurate in the sense of $L^2$ or $L^\infty$. The common message shared among these works is that learning the distribution is as easy (or hard) as learning the score function. More precisely, the estimation error of the target density is mainly controlled by the estimation error of the score function and  the discretization error of the diffusion processes (as another error source) scales at most polynomially in the data dimension.\mycomment{and hence does not incur the curse of dimensionality.} However, there has been relatively little work to address the problem of score estimation error.

More recently, it was proven in \citet{oko2023diffusion} that SGMs can estimate Besov densities with a minimax optimal rate under the total variation distance. However, the obtained sample complexity of density estimation over a Besov space suffers from the curse of dimensionality. The paper \citet{oko2023diffusion} further proved that the estimation rate can be substantially improved under the additional assumption that the data distribution is  supported on a low-dimensional linear subspace, in which case the resulting rate only depends on the intrinsic dimension. Distributions with the same low-dimensional structure was also studied by \citet{chen2023score} in the Lipschitz continuous setting. The paper \citet{de2022convergence} obtained convergence bounds in the Wasserstein distance for SGMs under a more general manifold hypothesis on the target distributions (including empirical measures).


Our work differs from the work above in that we do not make low-dimensional assumption on the data distribution. Instead, we assume that the target is absolutely continuous with respect to a Gaussian and that the log-relative-density belongs to the Barron space \citet{barron1993universal}. Barron functions have recently received much attention due to the fact that shallow networks can approximation them without curse of dimensionality; see, e.g., \citet{klusowski2018approximation,siegel2022sharp,ma2022barron}. In the context of generative modeling, the recent work \citet{domingo2021energy,domingo2021dual} investigated the statistical guarantees of energy-based models under the assumption that the underlying energy function lie in Barron space (or the $\mathcal{F}_1$ space therein). The work \citet{lee2017ability} obtained expressive bounds for normalizing flows in representing distributions that are push-forwards of a base distribution through compositions of Barron functions. This work shares the same spirit as \citet{domingo2021energy,domingo2021dual,lee2017ability} and demonstrates the statistical benefits of SGMs when target distribution exhibits low-complexity structures. 

We note that in an earlier version of this work, the log-relative density $f$ was assumed to be bounded. After discussions with an anonymous reviewer, we were inspired to strengthen the results to allow $f$ to grow at infinity. In more detail, the reviewer pointed out that when $f$ is bounded, the data distribution satisfies a log-Sobolev inequality (LSI) with constant $e^{\|f\|_{\infty}}$, which implies that the distribution can be sampled via Langevin dynamics with an estimator for the vanilla score. Our current results extend beyond the LSI case.

\subsection{Notation}
Throughout this article, we study functions and probability distributions on a Euclidean space $\R^d$ of a fixed dimension $d$. We let $\| \cdot \|$ denote the Euclidean norm on $\R^d$. For a vector or function, $\| \cdot \|_{\infty}$ denotes the supremum norm, and $\| \cdot \|_{Lip}$ denotes the Lipschitz seminorm of a metric space-valued function. We let $\gamma_d(dx)$ denote the standard Gaussian measure on $\R^d$, i.e., $\int f(x) \gamma_d(dx) = (2\pi)^{-d/2} \int_{\R^d} f(x) e^{-\|x\|^2/2} dx$. The indicator of a set $S$ is denoted $\mathbb{I}_S$. We denote by $( \cdot )^+$ the ReLU activation function, defined by $(c)^+ = \max(0,c)$ for $c \in \R$. For a vector $x$, $(x)^+$ is interpreted componentwise. For $X_0 \in \R^d$ and $t > 0$, we define $\Psi_t(\cdot|X_0)$ as the Gaussian density function with mean $e^{-t}X_0$ and variance $1-e^{-2t}$. For non-negative functions $g(x), h(x)$ defined on $\R^d$, we write $g(x) \lesssim h(x)$ (resp. $\gtrsim$) or $g(x) = O(h(x))$ (resp. $\Omega(h(x))$ if there exists a constant $C_d > 0$ which depends at most polynomially on the dimension $d$ such that $g(x) \leq C_d h(x)$ (resp. $\geq$). We write $g(x) = \Theta(h(x))$ if $g(x) \lesssim h(x) \lesssim g(x).$ For $\beta \in (0,1/2)$ and $g \in L^2(\gamma_d)$ we define $M_{\beta}(g) := \int_{\R^d} \left| g \left( \frac{x}{1-2\beta} \right) \right|^2 \gamma_d(du)$.

\section{Background}
In this section, we give a brief overview of the mathematical preliminaries required to understand our main result.

\subsection{A primer on SGMs}\label{sec:primer}
\mycomment{\textbf{Forward and reverse processes:}}In a score-based generative model (SGM), data is first transformed to noise via a forward process; we work with an Ornstein-Uhlenbeck process $(X_t)_{0 \leq t \leq T}$ , which solves the stochastic differential equation
\begin{equation}\label{fwd}
  dX_t = -X_t dt + \sqrt{2} dW_t, \; X_0 \sim p_0. 
\end{equation}
Here, $p_0$ is some initial data distribution on $\R^d$, and $W_t$ denotes $d$-dimensional Brownian motion. In the limit $T \rightarrow \infty$, the law $p_t$ of the process $X_T$ quickly approaches that of the standard Gaussian measure $\gamma_d$; in particular, we have $p_t \rightarrow \gamma_d$ exponentially fast in $t$ in the KL divergence, total variation metric and 2-Wasserstein metric \citet{bakry2014analysis, villani2021topics}. The SDE can be solved explicitly and, at each fixed time $t$, the solution coincides in law with the random variable
$$ X_t = e^{-t} X_0 + \sqrt{1-e^{-2t}}\xi, \; \; X_0 \sim p_0, \; \xi \sim N(0,I_d).
$$
In particular, $X_t$ conditioned on $X_0$ is a Gaussian random variable with mean $e^{-t}X_0$ and variance $1-e^{-2t}$. By averaging over $X_0$, we obtain a simple expression for the density $p_t$ in terms of $p_0$:
\begin{equation}\label{pt}
    p_t(x) = Z_t^{-1} \int \exp\left( -\frac{(x-e^{-t}y)^2}{2(1-e^{-2t})} \right) dp_0(y)
\end{equation}
where $Z_t = (2\pi (1-e^{-2t}))^{d/2}$ is the time-dependent normalization constant.

The reverse process, $\bar{X}_t = X_{T-t}$, is also a diffusion process \citet{anderson1982reverse, haussmann1986time}, solving the SDE (for $0 \leq t \leq T)$
\begin{equation}\label{rev}
    \bar{X}_t = (\bar{X}_t + 2 \nabla \log p_{T-t}(\bar{X}_t)) dt + \sqrt{2} d\bar{W}_t, \; \bar{X}_0 \sim p_t,
\end{equation}
where $\bar{W}_t$ denotes time-reversed Brownian motion on $\R^d$. In order to implement SGMs in practice, one must discretize the OU process, and a canonical algorithm is the exponential integrator scheme (EIS). In order to implement the EIS, one samples $\bar{X_0}^{dis} \sim p_T$, picks time steps $0 = t_0 < t_1 < \dots < t_N \leq T$ and simulates the SDE
$$ d\bar{X}_{t}^{dis} = \left(\bar{X}_t^{dis} + 2 \nabla \log p_{T-t_k}(\bar{X}_{t_k}) \right)dt + d\bar{B}_t
$$
for each interval $[t_k,t_{k+1}]$. S Also, the reverse process is typically initialized at $\bar{X}_0 \sim \gamma_d$ in practice, because $p_T$ is unknown. However, the error accrued by this choice is small, evidenced by the exponential convergence of $p_T$ to $\gamma_d$ as $T \rightarrow \infty$. The process one samples is then obtained by replacing the score function at time $T-t_k$ with a score estimate $\textbf{s}_k$:
\begin{equation}\label{sgmrevproc}
    \begin{cases}
        dY_t = \left(Y_t + 2 \textbf{s}_k(Y_{t_k}) \right)dt + dB_t, \; t \in [t_k,t_{k+1}] \\
        Y_0 \sim N(0,Id).
    \end{cases}
\end{equation}

\textbf{Loss function:} To learn the score function at time $t$, a natural objective to minimize is the following least-squares risk, namely, 
\begin{align*}
    \textbf{s}_t(t,X) \mapsto \ \E_{X_t \sim p_t} \left[\|s(t,X_t) - \nabla_x \log p_t(X_t)\|^2\right],
\end{align*}
for a given estimator $\textbf{s}_t: \R^d \rightarrow \R^d$. However, this risk functional is intractable since, in the generative modeling setting, one does not have access to pointwise data of the score function. However, it can be shown \citet{vincent2011connection} that for any $\textbf{s}_t$,
\begin{align*}
\E_{X_t \sim p_t} [\|\textbf{s}_t(X_t) - \nabla_x \log p_t(X_t)\|^2] = \E_{X_0 \sim p_0} \left[ \E_{X_t \sim p_t | X_0} [\|\textbf{s}_t(t,X_t) - \Psi_t(X_t|X_0) \|^2] \right] + E,
\end{align*}
where $E$ is a constant independent of $\textbf{s}_t$. Here, $\Psi_t(X_t|X_0) = - \frac{X_t - e^{-t}X_0}{1-e^{-2t}}$ denotes the score function of the forward process conditioned on the initial distribution. Note that the integral on the right-hand side can be approximated on the basis of samples of $p_0$, since the trajectories $(X_t | X_0)$ are easy to generate. This motivates our definition of the population risk at time $t$:
\begin{equation}\label{poprisk}
    \mathcal{R}^t_t(\textbf{s}_t) = \E_{X_0 \sim p_0} \left[ \E_{X_t \sim p_t | X_0} [\|s(t,X_t) - \Psi_t(X_t|X_0) \|^2] \right].
\end{equation}
If we define the individual loss function at time $t$ by
\begin{equation}\label{loss}
    \ell^t_t(\textbf{s}_t,x) = \E_{X_t | X_0 = x} \left[\|\textbf{s}_t(X_t) - \Psi_t(X_t|X_0)\|^2 \right],
\end{equation}
then the population risk can be written as
\begin{align*}
    \mathcal{R}^t_t(\textbf{s}_t) = \E_{x \sim p_0} [\ell^t_t(\textbf{s}_t,x)].
\end{align*}
We also define the empirical risk associated to the i.i.d. samples $\{X_i\}_{i=1}^{N}$ by
\begin{equation}
    \widehat{\mathcal{R}^t}_t^N(\textbf{s}_t) = \frac{1}{N} \sum_{i=1}^{N} \ell^t_t(\textbf{s}_t,X_i).
\end{equation}
We then to solve the optimization problem
\begin{align*}
    \min_{\textbf{s}_t \in \mathcal{F}} \widehat{\mathcal{R}^t}_t^N(\textbf{s}_t)
\end{align*}
where $\mathcal{F}$ is an appropriately defined class of vector fields on $\R^d$ (e.g., a class of neural networks). The use of the risk functional in \ref{poprisk} to learn the score has been termed \textit{denoising score matching}, because the function $\Psi_t(X_t|X_0)$ is the noise added to the sample $X_0$ along the forward process.

\subsection{Neural networks}\label{NNbackground}
A \textit{fully connected feedforward ReLU neural network} is a function of the form
$$ x \mapsto W_L \left(W_{L-1} \left( \cdots W_2 \left(W_1 x + b_1 \right)^{+} + b_2 \dots \right)^{+} + b_{L-1} \right)^{+}+b_{L},
$$
where $W_L: \R^{d_{L-1}} \times \R^{d_L}$ are matrices, $b_L \in \R^{d_L}$ are vectors, and, when $x$ is a vector, $(x)^{+}$ is interpreted as a componentwise mapping. The parameters $(W_i)_{i=1}^{L}$ and $(b_i)_{i=1}^{L}$ are called the \textit{weights} and \textit{biases} of the neural network respectively. The number of columns of $W_i$ is called the \textit{width} of the $i$th layer. When $L = 2$, a neural network is called \textit{shallow}, and, when they take values in $\R$, such networks admit a representation
$$
    x \mapsto \sum_{i=1}^{m} a_i (\langle w_i, x \rangle + b_i)^+,
$$
where $x, w_i \in \R^d$ and $a_i,b_i \in \R$. Neural networks have achieved remarkable success in learning complex, high-dimensional functions. In this work, we study their ability to express the score function of the diffusion process defined in Equation \ref{fwd}. In order to control the generalization error incurred by empirical risk minimization, one must introduce a notion of 'complexity' for neural networks, and a natural notion is the \textit{path norm}, defined for real-valued shallow ReLU neural networks by
$$ \|\phi\|_{\textrm{path}} := \sum_{i=1}^{m} |a_i|\left(\|w_i\|_1 + |b_i| \right), \; \; \phi(x) = \sum_{i=1}^{m} a_i (w_i^T x + b_i)^{(+)}
$$
and extended in a natural way to vector valued/deep ReLU neural networks. It has been shown that the collection of $L$-fold compositions of shallow networks of uniformly bounded path seminorm enjoys good generalization properties in terms of Rademacher complexity.

\section{Problem and main results}
We outline our main assumptions on the data distribution.
\begin{assumption}\label{datadistr2}
    The data distribution $p_0$ is absolutely continuous with respect to the standard Gaussian distribution. Throughout the paper, we let $p_0$ denote both the probability distribution and the PDF of the data, and we write $f(x) := \|x\|^2/2 + \log p_0(x)$ for the log-relative density of $p_0$ with respect to the standard Gaussian distribution, so that
$$ p_0(x) = \frac{1}{Z} \exp\Big(-\frac{\|x\|^2}{2} + f(x)\Big).
$$
We further assume that 
\begin{enumerate}
    \item There exist positive constants $r_f, \alpha, \beta$ with $\beta \ll 1$ such that $-\alpha \|x\|^2 \leq f(x) \leq \beta \|x\|^2$ whenever $\|x\| \geq r_f$, and $\alpha, \beta$ satisfy $c(\alpha,\beta) := \frac{4(\alpha+\beta)}{(1-\beta)} < 1$;
    \item $f$ is continuously differentiable;
    \item $\sup_{\|x\| \leq R} \|\nabla f(x)\|$ grows at most polynomially as a function of $R$;
    \item the normalization constant $Z$ satisfies $\frac{(2\pi)^{d/2}}{Z} \leq C$, where $C$ is a constant depending at most polynomially on dimension.
\end{enumerate}
\end{assumption}
Assumption \ref{datadistr2} roughly states that the tail of the data distribution should decay almost as quickly as a standard Gaussian. As an illustrating example, consider the Gaussian mixture model $q(x) \propto e^{-\frac{\|x-x_1\|^2}{2\sigma_{\min}^2}} + \frac{\|x-x_2\|^2}{2\sigma_{\max}^2}.$ If we write $f(x) = \|x\|^2/2 + \log q(x)$, then for any $\epsilon > 0,$ there exists $r_{\epsilon} > 0$ such that
$$ -\left( \frac{1-\sigma_{\min}^2+\epsilon}{2\sigma_{\min}^2} \right)\|x\|^2 \leq f(x) \leq \left( \frac{\sigma_{\max}^2-1 + \epsilon}{2\sigma_{\max}^2} \right) \|x\|^2, \;
$$
whenever $\|x\| \geq r_{\epsilon}$. This shows that Assumption \ref{datadistr2} applies to Gaussian mixtures, as long as the bandwidths are suitably constrained.

The benefit of expressing the data distribution in terms of the log-relative density is that it leads to a nice explicit calculation of the score function. In particular, Lemma \ref{score} states that the $j^{\textrm{th}}$ component of the score function of the diffusion process is given by
    \begin{align}\label{sc}
        (\nabla_x \log p_t(x))_j &= \frac{1}{1-e^{-2t}} \left( -x_j + e^{-t} \frac{F_t^j(x)}{G_t(x)} \right),
    \end{align}
    where $F_t^j(x) = \int_{\R^d} (e^{-t} x_j + \sqrt{1-e^{-2t}}u_j) e^{f(e^{-t}x + \sqrt{1-e^{-2t}}u)} \gamma_d(du)$ and $G_t(x) = \int e^{f(e^{-t}x + \sqrt{1-e^{-2t}}u)} \gamma_d(du)$.
The linear growth assumption ensures that the tail $p_0$ has similar decay properties to that of a Gaussian, of which we make frequent use. The other assumptions on $\nabla f$ and the normalization constant are stated for convenience.
    
In order to obtain tractable bounds on the estimation error of SGMs for learning such distributions, it is necessary to impose additional regularity assumptions on the function $f$.
\begin{assumption}\label{approxassump}[Learnability of the log-relative density]
    For every $\epsilon > 0$ and $R > 0$, there exists an L-layer ReLU neural network $\phi_f^{R,\epsilon}$ which satisfies
    $$ \sup_{\|x\| \leq R} |f(x) - \phi_f^{R,\epsilon}(x)| \leq R \epsilon.
    $$
    We denote by $\eta(\epsilon,R) = \|\phi_f^{R,\epsilon}\|_{\textrm{path}}$ the path norm of the approximating network as a function of $\epsilon$ and $R$.
\end{assumption}
We will generally abbreviate $\phi_f^{R,\epsilon}$ to $\phi_f$ in mild abuse of notation. Assumption \ref{approxassump} is weak because any continuous function can be approximated by neural networks to arbitrary precision on a compact set (\citet{cybenko1989approximation}). However, we are mainly interested in cases where $\phi_f$ generalizes well to unseen data; this corresponds to $\eta(\epsilon,R)$ growing mildly as $\epsilon \rightarrow 0$. 

\subsection{General results}
Our first result shows that, under Assumptions \ref{datadistr2} and \ref{approxassump}, the score function can be approximated \textit{efficiently} by a neural network, even in high dimensions. Our result helps to understand the massive success of deep learning-based implementations of SGMs used in large-scale applications. We denote by $\nn_{L,K}$ the set of neural networks from $\R^d$ to $\R^d$ with depth $L$ and path norm at most $K$. Recall that for a class of vector fields $\f$, we denote by $\f^{\textrm{score},t} = \{x \mapsto \frac{1}{1-e^{-2t}}\left( -x + e^{-t} f(x)\right): f \in \f \}$.
\begin{prop}[Approximation error for score function]
    Suppose assumptions \ref{datadistr2} ad \ref{approxassump} hold. Then there exists a class of neural networks $\nn$ with low complexity such that
     \begin{align*} &\inf_{\phi \in \nn^{\textrm{score,t}}} \int_{\R^d} \|\phi(x) - \nabla_x \log p_t(x)\|^2 p_t(x) dx \\ &= O\left( \max \left( \frac{1}{(1-e^{-2t})^2} (1+2\alpha)^{2d}  \epsilon^{2(1-c(\alpha,\beta))}, \frac{1}{(1-e^{-2t})^3} \epsilon^{1/2} \right) \right),
    \end{align*}
\end{prop}
The class of neural networks is defined precisely in the appendix. A high-level idea of the proof is to show that the map from $f$ to the score function $\nabla_x \log p_t$ has low-complexity in a suitable sense. The next result concerns the generalization error of learning the score function via empirical risk minimization.

\begin{prop}\label{informalgeneralization}[Generalization error for score function]
Set $R_{\epsilon} = \sqrt{d+\frac{1}{1-c(\alpha,\beta)}\log \left( \frac{(1+2\alpha)^d}{\epsilon t^2}\right)}$ and $\tilde{R}_{\epsilon} = \sqrt{d-\log(t^6 \epsilon^4)}.$ The empirical risk minimizer $\widehat{\textbf{s}}$ in $\nn^{score,t}$ of the empirical risk $\widehat{\mathcal{R}}^t$ satisfies the population risk bound
    $$ \mathcal{R}^t(\widehat{\textbf{s}}) = O(\epsilon^2),
    $$
    provided the number of training samples is $N \geq N_{\epsilon,t}$, where $N_{\epsilon,t}$ satisfies
    \begin{align*}N_{\epsilon,t} &= \Omega \Bigg( \max \Bigg(2^{2 L_f+10} d^2 t_0^{-6} (1+2\alpha)^{12d} \epsilon^{-4} \eta^{4}\left( R_{\epsilon}, (1+2\alpha)^{2d}\epsilon^{2(1-2c(\alpha,\beta))} \right) , \\ &2^{2L_f+10} (1+2\alpha)^{12d} t^{-6-72c(\alpha,\beta)} \epsilon^{-4-48c(\alpha,\beta)} \eta^4(\tilde{R}_{\epsilon},t^6 \epsilon^4) \Bigg).
    \end{align*}

\end{prop}

The next result describes how the learned score can be used to produce efficient samples from $p_0$, with explicit sample complexity rates.
\begin{prop}\label{distributionestimationmain}[Distribution estimation error of SGMs]
    Let $\widehat{\textbf{s}}$ denote the empirical risk minimizer in $\nn^{score,t}$ of the empirical risk $\widehat{\mathcal{R}}^t$, let $\widehat{p}$ denote the distribution obtained by simulating the reverse process defined in Equation \ref{sgmrevproc} over the time interval $[t_0,T]$ with $\widehat{\textbf{s}_{t_k}}$ in place of the true score for each discretization point, using the exponential integrator discretization scheme outlined in Section \ref{sec:primer} with maximum step size $\kappa$ and number of steps $M$. Then, with $T = \frac{1}{2} \left( \log(1/d) + 2d \log (1+2\alpha) +2(1-c(\alpha,\beta)) \log(1/\epsilon) \right),$
    $M \geq d T \epsilon$, $\kappa \lesssim \frac{1}{M}$, $t_0 \leq M_{\beta}(f)^{-2} \epsilon^2$, and $N \geq N_{\epsilon,t_0}$ (where $N_{\epsilon,t_0}$ is as defined in Proposition \ref{informalgeneralization}), we have
    $$TV(p_0,\widehat{p}) = O\left( \epsilon \right)$$ with probability $> 1 - \textrm{poly}(1/N)$.
\end{prop}
The proof of the above result has three essential ingredients: Prop \ref{informalgeneralization}, which controls the score estimation error for any fixed $t$, an application of Theorem 1 in \citet{benton2023linear}, which bounds the KL divergence between $p_{t_0}$ and $\widehat{p}$ along the exponential integrator scheme in terms of $\kappa$, $M$, $T$ and the score estimation error, and Lemma \ref{KLderivative}, which proves that the KL divergence between $p_0$ and $p_{t_0}$ can be bounded by $M_{\beta}(f)$.

\subsection{Examples}
We now discuss several concrete examples to which our general theory can be applied.

\vspace{2mm}
\textbf{Infinite-width networks:} Suppose that $p_0 \propto e^{-\frac{\|x\|^2}{2}+f(x)}$, where $f(x)$ is an infinite-width ReLU network of bounded total variation, i.e., $f(x) = \int_{\R^{d+2}} a(w^Tx+b)^{(+)} d\mu(a,w,b)$, where $c > 0$ and $\mu$ is a probability measure on $\R^{d+2}$. For such an $f$, the \textit{Barron norm} is defined as
$$ \|f\|_{\mathcal{B}} := \inf_{\mu} \int_{\R^{d+1}} |a|(\|w\|_1 + |b|) \mu(da,dw,db),
$$
where the infimum is over all $\mu$ such that the integral representation holds. The space of all such functions is sometimes referred to as the Barron space or variation space associated to the ReLU activation function. The Barron space has been identified as the 'correct' function space associated with approximation theory for shallow ReLU neural networks, since direct and inverse approximation theorems hold. Namely, for any $f$ in the Barron space and any $R > 0$, there exists a shallow ReLU neural network $f_{NN}$ such that $\sup_{\|x|| \leq R} |f(x) - f_{NN}(x)| \leq R\epsilon$, where $f_{NN}$ has $O(\|f\|_{\mathcal{B}}^2 \epsilon^{-2})$ parameters and $\|f_{NN}\|_{\textrm{path}} \lesssim \|f\|_{\mathcal{B}}$. Conversely, any function which can be approximated to accuracy $\epsilon$ by a network with path norm uniformly bounded in $\epsilon$ belongs to the Barron space. A comprehensive study of Barron spaces can be found in \citet{ma2022barron}.

Under the Barron space assumption on $f$, we can leverage the linear growth and fast approximation rate of Barron spaces to obtain dimension-independent sample complexity rates for SGMs.


\begin{prop}\label{barronestimation}[Distribution estimation under Barron space assumption]
    Suppose that $p_0(x) \propto e^{-\|x\|^2/2 +f(x)}$, where $f$ belongs to the Barron space. Let $\|f\|_{\mathcal{B}}$ denote the Barron norm and let $c_f = \inf\{c > 0: |f(x)| \leq c\|x\|\} \leq \|f\|_{\mathcal{B}}.$ Given $\delta \in (0,1)$, let $\epsilon(\delta)$ be small enough that $\frac{8c_f}{\tilde{R}_{\epsilon}} \leq \delta,$ where $\tilde{R}_{\epsilon} = \sqrt{d-\log(M_{1+\delta/4}(f)^{-10} \epsilon^{-10})}$ Then for all $\epsilon \leq \epsilon_0$, the distribution $\widehat{p}$ learned by the diffusion model satisfies
    $$ TV(\widehat{p},p_0) = O(\epsilon),
    $$
    provided the number of samples $N$ satisfies.
    $$ N = \Omega \left( 2^{2L_f+10} \left(1+\frac{2c_f}{\tilde{R}_{\epsilon}} \right)^{12d} M_{1+\delta/4}^{12+144 \delta} \epsilon^{-16 - 192 \delta} \|f\|_{\mathcal{B}}^4 \right).
    $$
\end{prop}
When $\epsilon$ is small, we essentially require $\epsilon^{-10}$ samples to learn $p_0$ to accuracy $\epsilon$ (up to the prefactors) - this is a significant improvement to classical sample complexity bounds in high dimensions, wherein the rate typically tends to zero as $d \rightarrow \infty.$ We emphasize that the significance of our contribution is that the \textit{rate} of the sample complexity is independent of dimension, and we leave it as an open direction whether all prefactors can be improved to depend only polynomially on $d$.
\vspace{2mm}

\vspace{2mm}
\textbf{Gaussian mixtures:} Distributions that describe real-world data are often highly \textit{multimodal}, and a natural model for such distributions is the Gaussian mixture model; we assume the initial density to be of the form
$$ p_0 = \frac{1}{2} \left(\frac{1}{Z_1} \exp\left( -\frac{\|x-x_1\|^2}{2\sigma_{\min}^2} \right) + \frac{1}{Z_2} \exp \left( -\frac{\|x-x_2\|^2}{\sigma_{\max}^2} \right) \right),
$$
where $0 < \sigma_{\min}^2 \leq \sigma_{\max}^2$ are the bandwidths and $x_1, x_2 \in \R^d$ are the modes of the distribution. The results that follow can easily be adapted to mixtures with more than two components and with arbitrary weights, but we keep the setting as simple as possible to illustrate the results. Due to the growth condition imposed in Assumption \ref{datadistr2}, our theory cannot be applied to Gaussian mixtures with arbitrarily small bandwidths; this is discussed further in Appendix \ref{appendixdistributionestimation}. We prove the following distribution estimation result for Gaussian mixtures.
\begin{prop}\label{gaussmixtureestimation}[Distribution estimation for Gaussian mixture]
    Given $\epsilon > 0$, set $R_{\epsilon} = \sqrt{d+\frac{1}{1-c(\alpha,\beta)}\log \left( \frac{(1+2\alpha)^d}{\epsilon t^2}\right)}$ and $\tilde{R}_{\epsilon} = \sqrt{d-\log(t^6 \epsilon^4)}$. Let $p_0(x) = \frac{1}{2} \left(\frac{1}{Z_1} e^{-\frac{\|x-x_1\|^2}{2\sigma_{\min}^2}} + \frac{1}{Z_2} e^{-\frac{\|x-x_2\|^2}{2\sigma_{\max}^2}} \right)$ be a mixture of two Gaussians, and fix $\delta \ll 1$. Assume that the bandwidths $\sigma_{\min}^2, \sigma_{\max}^2$ satisfy $c(\alpha,\beta) = \frac{4(\alpha+\beta)}{1-2\beta} < 1$, where $\alpha$ and $\beta$ are as defined in Assumption \ref{datadistr2}. Then there exists an $\epsilon_0$ (depending on $\delta$) such that for any $\epsilon \leq \epsilon_0$, the distribution $\widehat{p}$ learned by the SGM satisfies 
    $$ TV(\widehat{p}, p_0) = O\left((1+2\alpha)^d \epsilon^{1-c(\alpha,\beta)} \right),
    $$
    provided the number of samples $N$ satisfies $N \geq \max(N_{\epsilon,1}, N_{\epsilon,2})$, where
    $$ N_{\epsilon,1} = \Omega \left( 2^{2 L_f+10} d^2 t_0^{-6} (1+2\alpha)^{12d} \epsilon^{-4} \cdot \sup_{\|x\| \leq R_{\epsilon}} p_0^{-4}(x) \right)
    $$
    and $$N_{\epsilon,2} =\Omega \left( 2^{2L_f+10} (1+2\alpha)^{12d} t^{-6-72c(\alpha,\beta)} \epsilon^{-4-48c(\alpha,\beta)} \cdot \sup_{\|x\| \leq \tilde{R}_{\epsilon}} p_0^{-4}(x) \right).$$
    As an example, if $\sigma_{\min}^2 = \sigma_{\max}^2 = 1$, then we have $$TV(\widehat{p},p_0) = O(\epsilon)$$ provided the number of samples satisfies
    $$ N = \Omega \left(2^{2L_f+10} (1+\delta)^{12d} e^{d/2} \epsilon^{-24 - 768\delta} \right).
    $$

\end{prop}
The details of the proof are presented in Appendix \ref{appendixdistributionestimation}. One technical detail is that we need to be able to approximate the log density by a ReLU neural network so that Assumption \ref{approxassump} is satisfied. Unlike in the previous example, the log-likelihood is \textit{not} a Barron function. However, it can be shown that for any $R$, the restriction of the log-likelihood to $B_R$ can be represented by a composition of two Barron functions, and from this it follows that the log-likelihood can be locally approximated by a ReLU network with two hidden layers.

\section{Conclusion}
In this paper, we derived distribution estimation bounds for SGMs, applied to a family of sub-Gaussian densities parameterized by Barron functions. The highlight of our main result is that the sample complexity independent of the dimension. An important message of this work is that, for a data distribution of the form in Assumption \ref{datadistr2}, a low-complexity structure of the log-likelihood $f$ induces a similar low-complexity structure of the score function. In particular, the score function can be approximated in $L^2$ by a neural network without the curse of dimensionality. Some recent works \citep{oko2023diffusion, chen2023score} have derived distribution estimation bounds under assumptions of \textit{low-dimensionality} on the data; we chose to investigate an approximation-theoretic notion of 'low-complexity', and thus our results are a complement to these existing works.

We conclude by mentioning some potential directions for future research. First, we wonder whether similar results could be achieved if we relax some of our assumptions on the data distribution; for instance, we would like to extend the results to the case where the log-likelihood is allowed to decay arbitrarily quickly. Second, it is not clear whether our estimation rate for the class of densities considered is sharp, and thus obtaining lower bounds for sampling from such densities is an interesting open problem. We conjecture that our generalization error bound can be improved using more refined techniques such as local Rademacher complexities. Finally, obtaining training guarantees of score-based generative models remains an important open problem, and another natural direction for future work would be to study the gradient flow/gradient descent dynamics of score matching under similar assumptions on the target distribution.

\section{Acknowledgements}
Frank Cole and Yulong Lu thank the support from the National Science Foundation through the award DMS-2343135.

\bibliography{refs}

\newpage
\appendix
\addcontentsline{toc}{section}{Appendix} 
\part{Appendix} 
{\parttoc} 

\section{Properties of the score function and process density}
Let us recall the setup of SGMs. We are given samples from a high-dimensional probability distribution $p_0$, and we wish to learn additional samples. We define the forward process $X_t$ as the solution to the SDE
$$
    \begin{cases}
        dX_t = -X_t dt + \sqrt{2} dW_t, \; 0 \leq t \leq T, \\
        X_0 \sim p_0,
    \end{cases}
$$
and we note that the marginal distribution $p_t$ of $X_t$ at time $t$ quickly approaches the standard normal distribution $\gamma_d$. The reverse process $\bar{X}_t = X_{T-t}$ happens to satisfy the SDE
\begin{equation}\label{revsde}
    \begin{cases}
        d\bar{X}_t = (\bar{X}_t + 2 \nabla_x \log p_{T-t}) dt + \sqrt{2} dW_t, \; 0 \leq t \leq T, \\
        \bar{X}_0 = X_T,
    \end{cases}
\end{equation}
and so to sample the data distribution $p_0$, we run the SDE in equation \ref{revsde}, but with $\bar{X}_0$ as a standard normal and with the score function $\nabla_x \log p_{T-t}$ replaced by an empirical estimator $\hat{\textbf{s}}(t,x)$. This is made possible by a technique known as score matching \citep{hyvarinen2005estimation,vincent2011connection}, which frames the score estimation as a supervised learning problem.

The key assumption of our analysis is that the data distribution $p_0$ is proportional to $e^{-\frac{\|x\|^2}{2} + f(x)}$, where $f$ has 'low-complexity' in the sense of Assumptions \ref{datadistr2} and \ref{approxassump}. Under this assumption, we can explicitly compute the score function and derive sub-Gaussian bounds on the forward process density.

\begin{lemma}\label{score}
    Let $\gamma_d(du)$ denote the standard Gaussian probability measure on $\R^d$. Then under Assumption \ref{datadistr2}, the $j^{\textrm{th}}$ component of the score function of the diffusion process is given by
    \begin{align}\label{sco}
        (\nabla_x \log p_t(x))_j &= \frac{1}{1-e^{-2t}} \left( -x_j + e^{-t} \frac{F_t^j(x)}{G_t(x)} \right),
    \end{align}
    where $F_t^j(x) = \int_{\R^d} (e^{-t} x_j + \sqrt{1-e^{-2t}}u_j) e^{f(e^{-t}x + \sqrt{1-e^{-2t}}u)} \gamma_d(du)$ and $G_t(x) = \int e^{f(e^{-t}x + \sqrt{1-e^{-2t}}u)} \gamma_d(du)$.
\end{lemma}
\begin{proof}[Proof of Lemma \ref{score}]
    The forward process density is given by
    \begin{align*}
        p_t(x) = \frac{1}{Z_t}\int e^{-\frac{-\|x-e^{-t}y\|^2}{2(1-e^{-2t})}} e^{-\|y\|^2/2 + f(y)} dy,
    \end{align*}
where $Z_t$ is the normalization constant. We therefore have
\begin{align*}
    (\nabla_x p_t(x))_j = \frac{1}{Z_t(1-e^{-2t})} \left( -x_j p_t(x) + e^{-t} \int y_j e^{-\frac{-\|x-e^{-t}y\|^2}{2(1-e^{-2t})}} e^{-\|y\|^2/2 + f(y)} dy \right)
\end{align*}
and thus
\begin{align*}
    (\nabla_x \log p_t(x))_j = \frac{1}{1-e^{-2t}} \left( -x_j  + e^{-t} \frac{\int y_j e^{-\frac{-\|x-e^{-t}y\|^2}{2(1-e^{-2t})}} e^{-\|y\|^2/2 + f(y)} dy}{\int e^{-\frac{-\|x-e^{-t}y\|^2}{2(1-e^{-2t})}} e^{-\|y\|^2/2 + f(y)} dy} \right).
\end{align*}
By completing the square, we have
\begin{align*}
    e^{-\frac{\|x-e^{-t}y\|^2}{2(1-e^{-2t})}} e^{-\|y\|^2/2} = e^{-\frac{\|y-e^{-t}x\|^2}{2(1-e^{-2t})}} e^{-\|x\|^2/2},
\end{align*}
and therefore, after cancellation and an appropriate change of variables, we arrive at
\begin{equation}
\begin{aligned}
    (\nabla_x \log p_t(x))_j &= \frac{1}{1-e^{-2t}} \left( -x_j  + e^{-t} \frac{\int y_j e^{-\frac{-\|y-e^{-t}x\|^2}{2(1-e^{-2t})}} e^{-\|x\|^2/2 + f(y)} dy}{\int e^{-\frac{-\|y-e^{-t}x\|^2}{2(1-e^{-2t})}} e^{-\|x\|^2/2 + f(y)} dy} \right) \\
    &=  \frac{1}{1-e^{-2t}} \left( -x_j  + e^{-t} \frac{\int y_j e^{-\frac{-\|y-e^{-t}x\|^2}{2(1-e^{-2t})}} e^{f(y)} dy}{\int e^{-\frac{-\|y-e^{-t}x\|^2}{2(1-e^{-2t})}} e^{f(y)} dy} \right) \\
    &=  \frac{1}{1-e^{-2t}} \left( -x_j  + e^{-t} \frac{\int (e^{-t}x_j+\sqrt{1-e^{-2t}}u_j)e^{f(e^{-t}x + \sqrt{1-e^{-2t}}u)} \gamma_d(du)}{\int e^{f(e^{-t}x + \sqrt{1-e^{-2t}}u)} \gamma_d(du)} \right) \\
    &:=  \frac{1}{1-e^{-2t}} \left( -x_j  + e^{-t} \frac{F_t^j(x)}{G_t(x)} \right).
\end{aligned}
\end{equation}
\end{proof}

The following pointwise sub-Gaussian bounds are used throughout this work.
\begin{prop}\label{subgauss}
    Under assumption \ref{datadistr2},for all $t > 0$, $\|x\| \geq r_f$,
    $$ p_t(x) \lesssim \left( 2\pi \left( 1-2 \beta \right)^{-1} \right)^{-d/2} e^{-\frac{\left(1-2 \beta \right)\|x\|^2}{2}}
    $$
\end{prop}
\begin{proof}
By completing the square as in Lemma \ref{score}, we have
    $$
    p_t(x) = \frac{1}{Z} (2\pi(1-e^{-2t}))^{-d/2} e^{-\|x\|^2/2} \int_{\R^d} e^{-\frac{\|x-e^{-t}y\|^2}{2(1-e^{-2t})}} e^{f(y)} dy.
    $$
    For $\|x| \leq r_f$, we then use the quadratic growth and a change of variables to bound $p_t$:
    \begin{align*}
        p_t(x) &= \frac{1}{Z} e^{-\|x\|^2/2} \int_{\R^d} e^{f(\sqrt{1-e^{-2t}}u + e^{-t}x)} \gamma_d(du) \\
        &\leq \frac{1}{Z} e^{-\|x\|^2/2} \int_{\R^d} e^{\beta\|\sqrt{1-e^{-2t}}u+e^{-t}x\|^2} \gamma_d(du) \\
        &\leq \frac{1}{Z} e^{-\|x\|^2/2} \int_{\R^d} e^{\beta \left(\|u\|^2+\|x\|^2 \right)} \gamma_d(du)\\ 
        &= \frac{\left( 1-2 \beta \right)^{-d/2}}{Z}e^{-\frac{(1-2\beta)\|x\|^2}{2}} \\
        &\lesssim \left( 2\pi \left( 1-2 \beta \right)^{-1} \right)^{-d/2} e^{-\frac{\left(1-2 \beta \right)\|x\|^2}{2}}
        \end{align*}
\end{proof}

The following lemma control the growth of $F_t^j$ and $G_t$.
\begin{lemma}\label{fggrowth}
    Let $F_t^j$ and $G_t$ be as defined in Lemma \ref{score}. Let $R \geq \max(r_f, \sqrt{\frac{1}{\beta} \sup_{\|x\| \leq r_f} |f(x)|})$. Then
    $$ \sup_{\|x\| \leq R} |F_t^j(x)| = O \left( \left(1-2\beta \right)^{-d/2} e^{\beta R^2} \right)
    $$
    and for $\|x\| \leq R, $$$
    \left(1+2\alpha \right)^{-d/2} e^{-\alpha R^2} \leq G_t(x) \leq \left(1-2\beta \right)^{-d/2} e^{-\beta R^2}.
    $$
\end{lemma}
\begin{proof}
For the lower bound on $G_t$, for $\|x\| \leq R$, $R$ sufficiently large, we have
\begin{align*}
    G_t(x) &= \int_{\R^d} e^{f(e^{-t}x + \sqrt{1-e^{-2t}}u)} \gamma_d(du) \\
    &\geq \int_{\R^d} e^{-\alpha \left\|e^{-t}x + \sqrt{1-e^{-2t}}u\right\|^2} \gamma_d(du) \\
    &\geq \int_{\R^d} e^{-\alpha \left( \|x\|^2 + \|u\|^2 \right)} \gamma_d(du) \\
    &\geq \left(1+2\alpha\right)^{-d/2} e^{-\alpha R^2}.
\end{align*}
The upper bound on $G_t$ is proven similarly. For the upper bound on $F_t^j$, the proof is similar: we have for $\|x\| \geq R$, 
\begin{align*}
    |F_t^j(x)| &= \left| \int_{\R^d} \left( e^{-t}x_j + \sqrt{1-e^{-2t}}u_j \right) e^{f(e^{-t}x+\sqrt{1-e^{-2t}}u)} \gamma_d(du) \right| \\
    &\leq \int_{\R^d} \left( \|x\| + \|u\| \right) e^{\beta \left( \|x\|^2 + \|u\|^2 \right)} \gamma_d(du) \\
    &\leq e^{\beta R^2} \left(\int_{\R^d} e^{\frac{C_f\|u\|^2}{2}} \gamma_d(du) + \int_{\R^d} \|u\| e^{\beta \|u\|^2} \gamma_d(du) \right) \\
    &= O \left( \left(1-2 \beta \right)^{-d/2} e^{\beta R^2} \right)
\end{align*}
\end{proof}

\section{High-level proof sketch}
Before delving into the details of the proof, we give an overview of the proof technique.

\subsection{Proof overview of approximation error bound}
Recall that the $j^{\textrm{th}}$ component of the score function takes the form $(t,x) \mapsto \frac{1}{1-e^{-2t}} \left( -x_j + e^{-t} \frac{F_t^j(x)}{G_t(x)} \right)$ where $F_t^j(x) = \int (e^{-t}x_j + \sqrt{1-e^{-2t}}u_j)e^{f(e^{-t}x+\sqrt{1-e^{-2t}}u)} \gamma_d(du)$ and $G_t(x) = \int e^{f(e^{-t}x+\sqrt{1-e^{-2t}}u)} \gamma_d(du)$. We break our approximation argument into two main steps.

\textbf{Step 1-  approximation of $\frac{F_t^j(x)}{G_t^j(x)}$ on the ball $B_R$}: The first step of the proof is to approximate function $\frac{F_t^j(x)}{G_t^j(x)}$ on a bounded domain. For the function $F_t^j$, note that the integrand can be viewed as a composition of three simple functions, namely the function $x \mapsto f(x)$, the one-dimensional exponential map $x \mapsto e^{x}$, and the two-dimensional product map $(x,y) \mapsto xy$. By assumption, $f$ can be approximated by a shallow neural network $\phi_f$ with low path-norm. It is well known that the latter two maps can be approximated on bounded domains by a shallow neural networks $\phi_{exp}$ and $\phi_{prod}$; see the Appendix for details. We note that the complexity of the neural network needed to approximate the exponential map $x \mapsto e^x$ on the interval $[-C,C]$ grows exponentially with $C$. However, the fast tail decay of the data distribution ensures that we can restrict attention to an interval which is not too large and still achieve good approximation bounds.

In turn, this allows us to approximate the integrand in the definition of $F_t^j$ by a deep neural network $x \mapsto \phi_{prod}(x_j, \phi_{\exp}(\phi_f(x)))$. We then discretize the integral with respect to $\gamma_d(du)$ using a Monte Carlo sampling argument, and the resulting neural network is
\begin{align*}
    \Phi_{F,t}^j(x) = \frac{1}{m} \sum_{i=1}^{m}  \phi_{prod}(e^{-t}x_j + \sqrt{1-e^{-2t}} (u_i)_j, \phi_{\exp}(\phi_f(e^{-t}x + \sqrt{1-e^{-2t}}u_i))),
\end{align*}
for Monte Carlo sample points $\{u_i\}_{i=1}^{m}$. The procedure for $G_t(x)$ works very similarly, since $G_t(x) = \int e^{f(e^{-t}x + \sqrt{1-e^{-2t}}u)}\gamma_d(du).$ This gives us neural networks $\Phi_{F,t}^j$ and $\Phi_{G,t}$ that approximate $F_t^j$ and $G_t$ on the ball $B_R$. By approximating the quotient map $(x,y) \mapsto \frac{x}{y}$ by another shallow network $\phi_{quot}$, we obtain a deep neural network approximation $\Phi_t^j := \phi_{quot}(\Phi_{F,t}^j(x),\Phi_{G,t}(x))$ for $\frac{F_t^j(x)}{G_t^j(x)}$, valid on $B_R$. 

Here, it is crucial that our approximation for the individual functions $F_t^j, G_t$ is in the sup norm; this ensures that the neural network approximation to $G_t$ is bounded away from zero (since $G_t$ itself is bounded away from 0), which in turn allows us to control the approximation error of the quotient map $(x,y) \mapsto \frac{x}{y}$.

\textbf{Step 2 - approximation on the unbounded domain $L^2(p_t)$, for fixed $t$}: The approximation metric we care about is ultimately not the uniform metric on $B_R$, but the $L^2(p_t)$ metric on all of $\R^d$. To deal with the unbounded approximation domain, we bound the tail of the density $p_t$ and use a truncation argument; in particular, $p_t$ is sub-Gaussian by Lemma \ref{subgauss}, so that the truncation error depends mildly on the radius of the ball $B_R$ from Step 1. By choosing the optimal $R$, this gives an approximation of $\frac{F_t^j}{G_t}$ in $L^2(p_t)$ for a fixed time, thus completing the second step of the proof.

\subsection{Proof overview of generalization}
Recall that our goal is, for each $t$, to bound the population risk $\mathcal{R}^t(\cdot)$ at the minimizer of the empirical risk over a class of neural networks (to be specified in the detailed proof). For technical reasons, we work with the minimizer $\widehat{\textbf{s}}$ of a truncated version $\widehat{\mathcal{R}}_R^{t}$ of the empirical risk, where the data is assumed to be uniformly bounded along the forward process. However, we choose the truncation radius $R$ large enough so that the error incurred by this step is marginal. If we define $\mathcal{R}^t_R$ as the corresponding $R$-truncated version of the population risk, then the generalization error can be decomposed as
$$ \mathcal{R}^t(\widehat{\textbf{s}}) = \underbrace{\left(\mathcal{R}^t(\widehat{\textbf{s}}) - \mathcal{R}^t_R(\widehat{\textbf{s}}) \right)}_{\textrm{truncation error}} + \underbrace{\left( \mathcal{R}^t_R(\widehat{\textbf{s}}) - \mathcal{R}^t_R(\textbf{s}^{\ast}_R) \right)}_{\textrm{generalization error}} + \underbrace{\left(\mathcal{R}^t_R(\textbf{s}^{\ast}_R) - \mathcal{R}^t(\textbf{s}^{\ast}) \right)}_{\leq 0} + \underbrace{\mathcal{R}^t(\textbf{s}^{\ast})}_{\textrm{approximation error}}.
$$
Here, $\textbf{s}^{\ast}$ is the minimizer of the population risk over the hypothesis class and $\textbf{s}^{\ast}_R$ is the minimizer of the $R$-truncated risk. The first term represents the error we create from working with the truncated risk rather than the true risk, and we bound it using existing large deviation bounds on the OU process from \citet{oko2023diffusion}. Term II represents the generalization error of the truncated risk. In order to bound the Rademacher complexities of some relevant function classes, we need the individual loss function $\ell^t$ to have certain properties (such as boundedness and Lipschitz-continuity). Unfortunately, the loss fails to have these properties fail for the loss $\ell$, but they do hold for it's truncated counterpart $\ell^t_R$, and this is our primary motivation for working with the truncated risk rather than the true risk. In turn, it allows us to apply existing generalization results for neural networks with bounded complexity to our setting. 

Term 3 is actually non-negative, because $\mathcal{R}^t_R(\cdot) \leq \mathcal{R}^t(\cdot)$ for any $R$, and hence $\min \mathcal{R}^t_R(\cdot) \leq \min \mathcal{R}^t(\cdot).$ Term 4 is the approximation error, and thus it will be of order $\epsilon^2$, provided the complexity of the hypothesis class is scaled as a suitable function of $\epsilon$. The proof concludes by balancing the truncation radius $R$ and sample size $N$ as suitable functions of $\epsilon.$

\subsection{Proof overview of distribution estimation}
There are several existing works that bound the distribution estimation error of score-based generative models in terms of the generalization error of the score function. We apply results from \citet{benton2023linear} which state that if the score generalization error is $O(\epsilon^2)$ then, provided the parameters of the sampling algorithm are chosen accordingly and the reverse process is stopped at time $T-t_0$, the learned distribution $\widehat{p}$ satisfies $TV(\widehat{p},p_{t_0}) = O(\epsilon)$, where $p_{t_0}$ is the distribution of the forward process at time $t_0 \ll 1$. All that remains is to bound $TV(p_{t_0},p_t)$ and choose $t_0$ as an appropriate function of $\epsilon$. For this, Lemma \ref{KLderivative} shows that $TV(p_{t_0},p_0) = O(\sqrt{t_0}).$

\section{Proofs for approximation}
The following lemma is a general approximation error bound in the sup norm for functions defined by certain integral representations. We use it to approximate the functions $F_t^j$ and $G_t$, since they admit natural integral representations. The proof adapts the proof of Theorem 1 in \citet{klusowski2018approximation}. Though the proof idea seems well known, we have not found the general version of the result in the literature, so we state it here in case it may be of interest to others.
\begin{lemma}\label{empproc}
    Let $g: \R^d \times \R^p \rightarrow \R$ be a function such that for all $K,R > 0,$
    \begin{enumerate}
        \item $L_{K,R} := \sup_{\|\theta\| \leq K} \|g(\cdot, \theta) \|_{Lip(B_R)} < \infty$,
        \item $C_{K,R} := \sup_{\|\theta\| \leq K , \; \|x\| \leq R}|g(x,\theta)| < \infty.$
    \end{enumerate}
    Here, $\| \cdot \|_{Lip(B_R)}$ is the Lipschitz constant of a function when restricted to the ball of radius $R$. Let $f: \R^d \rightarrow \R$ be a function of the form
    $$ f(x) = \int_{\|\theta\| \leq K} g(x,\theta) \mu(d\theta),
    $$
    where $\mu$ is a Radon measure on $\{\|\theta\| \leq K \}$. Then, for any $R > 0$, there exist $(a_i,\theta_i)_{i=1}^{n}$ with $a_i \in \{\pm 1\}$ and $\|\theta_i\| \leq K$ such that
    $$ \sup_{\|x\| \leq R} \left| f(x) - \frac{\|\mu\|_{TV}}{m} \sum_{i=1}^{m} a_i g(x,\theta_i) \right| \leq \frac{24 \|\mu\|_{TV}}{\sqrt{m}} \cdot C_{K,R} \cdot \sqrt{d \log(m L_{K,R} R)},
    $$
\end{lemma}

\begin{proof}
    Notice that by normalizing $\mu$ and decomposing it into positive an negative parts, we can write
    $$ f(x) = \|\mu\|_{TV} \int_{\{\pm1\} \times \{\|\theta\| \leq K\}} a g(x,\theta) \tilde{\mu}(da,d\theta),
    $$
    where $\tilde{\mu}$ is a probability measure. Let $(a_i,\theta)_{i=1}^{m}$ be an i.i.d. sample from $\tilde{\mu}^m$, and define $f_m(x;\Theta) = \frac{\|\mu\|_{TV}}{m} \sum_{i=1}^{m} a_i g(x,\theta_i)$. We can view $f_m(x;\Theta)$ as an empirical process on the parameter space $\Theta = (a_i,\theta_i)_{i=1}^{m}$ indexed by $x \in B_R$. Let $\widehat{\mu}$ denote the empirical measure associated to the samples, and notice that for any $x, x' \in B_R$, it holds that
    $$ \|x-x'\|_{L^2(\widehat{\mu})}^2 \leq \frac{1}{m} \sum_{i=1}^{m} |g(x,\theta_i) - g(x',\theta_i)|^2 \leq L_{K,R}^2 \|x-x'\|^2.
$$
This proves that the covering number of $B_R$ under the $L^2(\widehat{\mu})$ norm is $O(L_{K,R}(R/\epsilon)^{d})$. By Dudley's Theorem for empirical processes, it holds that
$$ \E \sup_{\|x\| \leq R} |f(x) - f_m(x;\Theta)| \leq \frac{24 \|\mu\|_{TV}}{\sqrt{m}} \inf_{0 \leq t \leq \delta/2} \int_{t}^{\delta/2} \sqrt{\log \mathcal{N}(B_R, \| \cdot \|_{L^2(\widehat{\mu}})}, \epsilon) d\epsilon,
$$
where $\delta = \sup_{\|x\| \leq R} \|x\|_{L^2(\widehat{\mu})}$. We can bound $\delta$ by $C_{K,R}$, and by choosing $t = O(1/m)$ in the infimum in the bound from Dudley's theorem, we find that
\begin{align*}
     \E \sup_{\|x\| \leq R} |f(x) - f_m(x;\Theta)| &\leq \frac{24 \|\mu\|_{TV}}{\sqrt{m}} \inf_{0 \leq t \leq C_{K,R}} \int_{t}^{C_{K,R}} \sqrt{\log(L_{K,R})+d\log(R/t)} dt \\
     &= O \left( \frac{\|\mu\|_{TV}}{\sqrt{m}} \cdot C_{K,R} \cdot \sqrt{\log(L_{K,R}) + \log(Rm)}  \right) \\
\end{align*}
\end{proof}

The following lemmas prove that the functions $F_t^j$ and $G_t$ can be well-approximated by neural networks. An important feature of the result is that the approximation is done uniformly on the ball of radius $R$, as opposed to in $L^2$. The main idea of the proof is to discretize the Gaussian integrals defining $F_t^j$ and $G_t$ via Lemma \ref{empproc}, and then exploit the compositional structure of the integrands.

\begin{lemma}\label{approxonballg}
    For any $t \in (0,\infty)$ and $1 \leq j \leq d$, let $G_t(x)$ be as defined in Lemma \ref{score}. Then for any $\epsilon > 0$ and any $R \geq r_f,$ there exists a neural network $\phi_{t,G}(x)$ such that
    $$ \sup_{\|x\| \leq R} \left| \phi_{t,G}(x) - G_t(x) \right| = O \left(e^{\beta R^2} \left(\epsilon + (1-2\beta)^{-d/2} e^{-\frac{(1-2\beta)R^2}{2}}\right) \right), 
    $$
    In addition, it holds that $\|\phi_{G,t}\|_{\textrm{path}} = O\left( e^{\beta R^2} \cdot \eta (R,\epsilon)\right)$.
\end{lemma}
\begin{proof}
    Recall that
    $$ G_t(x) := \int e^{f(e^{-t}x+\sqrt{1-e^{-2t}}u)} \gamma_d(du).
    $$
    Fix $R > 0$. We break $G_t$ into two parts:
    \begin{align*}
        G_t(x) &= \int e^{f(e^{-t}x+\sqrt{1-e^{-2t}}u)} \gamma_{d,R}(du) + \int_{\|u\| > R}  e^{f(e^{-t}x+\sqrt{1-e^{-2t}}u)} \gamma_d(du) \\
        &:= G_{t,R}(x) + \mathcal{E}_{G,R}(x),
    \end{align*}
    where $\gamma_{d,R}(du) = \mathbb{I}_{\|u\| \leq R}(u) \gamma_d(du)$. To proceed, we approximate the local part $G_{t,R}(x)$ of $G_t(x)$ and then control the error term $\mathcal{E}_{G,R}(x)$ using the tail decay of the data distribution. We apply Lemma \ref{empproc} with $g(x,u) = \exp(f(e^{-t}x + \sqrt{1-e^{-2t}}u))$ and $\mu(du) = \gamma_{d,R} (du)$. In this case, we have $\|\mu\|_{TV} \leq 1$, 
    $$ \sup_{\|x\| \leq R, \|u\| \leq R} g(x,u) \leq e^{2 \beta R^2} \; \; \; (\textrm{by the choice of $R$}),
    $$
    and
    $$ \sup_{\|u\| \leq R} \|g(\cdot,u) \|_{Lip(B_{R})} \leq \|e^f\|_{Lip(B_{\sqrt{2}R})} \leq e^{2 \beta R^2} \sup_{\|x\| \leq R} \|\nabla f(x)\| := e^{2 \beta R^2} D_{f,R}.
    $$
    The conclusion of Lemma \ref{empproc} states that there exist $u_1, \dots, u_m$ with $\sup_{1 \leq i \leq m} \|u_i\| \leq R$ such that
    \begin{align*}
        \sup_{\|x\| \leq R} \left| \int_{\R^d} e^{f(e^{-t}x + \sqrt{1-e^{-2t}}u)} \gamma_{d,R}(du) - \frac{\|\gamma_{d,R}\|_{TV}}{m} \sum_{i=1}^{m} e^{f(e^{-t}x + \sqrt{1-e^{-2t}}u_i)} \right| &= O \left(e^{2 \beta R^2} \sqrt{\frac{d}{m}} \right)\\
    \end{align*}
    Now, by Lemma \ref{approxef}, there exists a ReLU neural network $\phi_{f,exp}(x)$ such that 
    $$ \sup_{\|x\| \leq R} \left|\phi_{f,exp}(x) - e^{f(x)} \right| \lesssim e^{\beta R^2} \epsilon.
    $$
    and $\phi_{f,exp}$ satisfies $\|\phi_{f,exp}\|_{\textrm{path}} = O\left( \right).$  We define $\phi_{G,t}(x) := \frac{\|\gamma_{d,R}\|_{TV}}{m} \sum_{i=1}^{m} \phi_{f,exp}(e^{-t}x + \sqrt{1-e^{-2t}}u_i)$, which satisfies the approximation bound
    \begin{align*}
        \sup_{\|x\| \leq R} \left| \phi_{G,t}(x) - e^{f(x)} \right| &\leq  O \left(e^{2 \beta R^2} \sqrt{\frac{d }{m}} \right) + \sup_{\|x\| \leq R} \left|\phi_{f,exp}(x) - \frac{\| \gamma_{d,R}\|_{TV}}{m} \sum_{i=1}^{m} e^{f(e^{-t}x + \sqrt{1-e^{-2t}}u_i)} \right| \\
        &\leq  O \left(e^{2 \beta R^2} \sqrt{\frac{d}{m}} \right) +  \sup_{\|y\| \leq 3R} \left|\phi_{f,exp}(y) - e^{f(y)}\right| \\
        &= O \left( e^{2 \beta R^2} \sqrt{\frac{d }{m}} + e^{\beta R^2 }\epsilon \right).
    \end{align*}
    If we set $m = \Omega(\epsilon^{-2} d e^{2 \beta R^2})$, then
    $$  \sup_{\|x\| \leq R} \left| \phi_{G,t}(x) - e^{f(x)} \right| = O(e^{\beta R^2} \epsilon )
    $$
    It remains to bound the error term $\mathcal{E}_{G,R}(x)$:
    \begin{align*}
        \mathcal{E}_{G,R}(x) &= \int_{\|u\| > R} e^{f(e^{-t}x + \sqrt{1-e^{-2t}}u)} \gamma_{d}(du) \\
        &\leq \int_{\|u\| > R} e^{\beta(\|x\|^2 + \|u\|^2)} \gamma_d(du) \\
        &\leq e^{\beta R^2} \left(1-2 \beta\right)^{-d/2} e^{-\frac{(1-2\beta)R^2}{2}}.
    \end{align*}
    We conclude that
    $$ \sup_{\|x\| \leq R} \left| \phi_{t,G}(x) - G_t(x) \right| = O \left(e^{\beta R^2}\left( \epsilon +(1-2\beta)^{-d/2} e^{-\frac{(1-2\beta)R^2}{2}} \right)\right).
    $$
    To see the path norm bounds on $\phi_{G,t}$, we have
    \begin{align*} \|\phi_{G,t}\|_{\textrm{path}} &= \left\| \frac{\|\gamma_{d,R}\|_{TV}}{m} \sum_{i=1}^{m} \phi_{f,exp}(e^{-t}x + \sqrt{1-e^{-2t}}u_i) \right\|_{\textrm{path}} \\
    &\leq \frac{1}{m} \sum_{i=1}^{m} \|\phi_{f,exp}(e^{-t}x+\sqrt{1-e^{-2t}}u_i)\|_{\textrm{path}} \\
    &\leq (1+2R) \|\phi_{f,\exp}\|_{\textrm{path}} \; \; \; \textrm{(Behavior of path norm under scaling/translation)} \\
    &\leq (1+2R) \cdot e^{\beta R^2} \cdot \eta(R,\epsilon) \; \; \; \textrm{(path norm of $\phi_{f,\exp}$)} \\
    &= O \left(e^{\beta R^2} \eta(R,\epsilon) \right)
    \end{align*}
    This proves the claim.
\end{proof}

\begin{lemma}\label{approxonballf}
        For any $t \in (0,\infty)$ and $1 \leq j \leq d$, let $F_t^j(x)$ denote be as defined in Lemma \ref{score}. Then for any $\epsilon > 0$ and any $R \geq r_f,$ there exists a neural network $\phi_{t,F}^j(x)$ such that
        $$ \sup_{\|x\| \leq R} \left| \phi^j_{t,F}(x) - F^j_t(x) \right| = O \left(e^{\beta R^2} \left(\epsilon + (1-2\beta)^{-d/2} e^{-\frac{(1-2\beta)R^2}{2}}\right) \right), 
    $$
    In addition, it holds that $\|\phi^j_{F,t}\|_{\textrm{path}} = O\left( e^{3 \beta R^2} \cdot \eta (R,\epsilon)\right)$.
\end{lemma}
\begin{proof}
    Recall that
    $$ F_t^j(x) = \int_{\R^d} h^j(e^{-t}x+ \sqrt{1-e^{-2t}}u) \gamma_d(du),
    $$
    where $h^j(y) = y_j e^{f(y)}$. The proof is very similar to that of Lemma \ref{approxonballg}. Define local and global parts $F_{t,R}^j(x)$ and $\mathcal{E}_{F,R}^j(x)$ of $F$ to as in the proof of Lemma \ref{approxonballg} (but with $G_t$ replaced by $F_t^j$). Noting that $\sup_{\|u\|, \|x\| \leq R} h^j(e^{-t}x+\sqrt{1-e^{-2t}}u) \leq \sqrt{2}R e^{2\beta R^2}$ and that $h^j$ is $(R e^{2 \beta R^2} D_{f,R} + e^{2\beta R^2}) = O(e^{2\beta R^2})$-Lipschitz on $B_{\sqrt{2} R}$, we can apply Lemma \ref{empproc} to guarantee the existence of $u_1, \dots, u_m \in B_{\tilde{R}}$ such that
    \begin{align*}
        \sup_{\|x\| \leq R} \left| F_{t,R}^j(x) - \frac{\|\gamma_{R,d}\|_{TV}}{m} \sum_{i=1}^{m} h^j(e^{-t}x + \sqrt{1-e^{-2t}}u) \right| \lesssim  O \left(e^{2\beta R} \sqrt{\frac{d}{m}} \right).
    \end{align*}
    Lemma \ref{approxprodef} guarantees the existence of a neural network $\Phi_f^j(x)$ such that
$$ \sup_{\|x\| \leq R} \left| \Phi_f^j(x) - h^j(x) \right| \lesssim e^{\beta R^2} \epsilon. 
$$
If we define the ReLU network $\phi_{F,t}^j(x) := \frac{\|\gamma_{R,d}\|_{TV}}{m} \sum_{i=1}^{m} \Phi_f^j(e^{-t}x + \sqrt{1-e^{-2t}}u)$, then an application of the triangle inequality shows that
$$ \sup_{\|x\| \leq R} \left| F_{t,R}^j(x) - \phi_{F,t}^j(x) \right| = O \left( \sqrt{\frac{d}{m}}e^{2 \beta R^2} + e^{\beta R^2} \epsilon \right) = O(e^{\beta R^2} \epsilon),
$$
for $m = \Omega(\epsilon^{-2} d e^{2 \beta R^2})$. To conclude, we bound the error term $E_{F,R}(x)$ for $\|x\| \geq R$, as in Lemma \ref{approxonballg}:
\begin{align*}
    E_{F,R}(x) &= \int_{\|u\| > \tilde{R}} (e^{-t} x_j + \sqrt{1-e^{-2t}} u_j) e^{f(e^{-t}x + \sqrt{1-e^{-2t}}u)} \gamma_d(du) \\
    &\lesssim \int_{\|u\| > R} (e^{-t R + \sqrt{1-e^{-2t}}|u_j|})e^{\beta(\|x\|^2+\|u\|^2)} \\
    &= O \left(e^{\beta R^2} (1-2 \beta)^{-d/2} e^{-\frac{(1-2\beta)R^2}{2}} \right).
\end{align*}
This gives the result. The path norm bound follows a similar argument to that in Lemma \ref{approxonballg} and uses the path norm bound for the neural network approximant for $x \mapsto x_j e^{f(x)},$ proved in Lemma \ref{approxprodef}.
\end{proof}

\begin{lemma}\label{approxratio}
    Let $t \in [0,\infty)$, and let $F_t^j$ and $G_t$ be defined as in Lemma \ref{score}. Let $\epsilon > 0$ be small enough and $R > 0$ large enough. Then there exists a ReLU neural network $\phi^j_{t,F,G}$ such that
    $$ \sup_{\|x\| \leq R} \left| \phi^j_{t,F,G}(x) - \frac{F_t^j(x)}{G_t(x)} \right| = O \left(  (1+2\alpha)^{d} (1-2\beta)^{-d/2} e^{2(\alpha+\beta)R^2} \left(\epsilon + (1-2\beta)^{-d/2} e^{-\frac{(1-2\beta)R^2}{2}} \right) \right).
    $$
    In addition, we have $$\|\phi_{t,F,G}^j\|_{\textrm{path}} = O\left((1+2\alpha)^{3d} (1-2\beta)^{-3d/2} e^{6(\beta + \alpha)R^2} \eta(R,\epsilon) \right).$$

\end{lemma}

\begin{proof}
    Throughout this proof let us denote $F_t^j(x)$ by $F(x)$, $G_t(x)$ by $G(x)$, and $\phi_{t,F,G}^j(x)$ by $\phi_{F,G}(x)$, since none of the estimates will depend on $t$ or $j$. Recall from Lemma \ref{approxonballf} and Lemma \ref{approxonballg} that there exist ReLU neural networks $\phi_F$ and $\phi_G$ which approximate $F$ and $G$ on the ball of radius $R$ to error
    $$\left(e^{\beta R^2} \left(\epsilon + (1-2\beta)^{-d/2} e^{-\frac{(1-2\beta)R^2}{2}}\right) \right)
$$
with respect to the uniform norm, and that in addition the networks satisfy $\|\phi_F\|_{\textrm{\textrm{path}}} = O(e^{3\beta R^2} \eta(R,\epsilon))$ and $\|\phi_G\|_{\textrm{\textrm{path}}}= O(e^{\beta R^2} \eta(R,\epsilon))$ The proof proceeds in two steps:
\begin{enumerate}
    \item Show that $\frac{\phi_F}{\phi_G}$ approximates $\frac{F}{G}$ on the ball of radius $R$;
    \item Show that $\frac{\phi_F}{\phi_G}$ can be approximated by a neural network $\phi_{F,G}$ on the ball of radius $R$.
\end{enumerate}
Notice that
\begin{align*}
    \sup_{\|x\| \leq R} \left| \frac{F(x)}{G(x)} - \frac{\phi_F(x)}{\phi_G(x)} \right| &\leq \sup_{\|x\| \leq R} \left| \frac{1}{G(x)} \right| \cdot \sup_{\|x|\ \leq R} \left| F(x) - \phi_F(x) \right| \\ &+ \sup_{\|x\| \leq R} \left| \frac{\phi_F(x)}{G(x) \phi_G(x)} \right| \cdot \sup_{\|x\| \leq R} \left| G(x) - \phi_G(x) \right|.
\end{align*}
By Lemma \ref{fggrowth}, we have for $\|x\| \leq R$ that $F$ and $G$ satisfy the bounds
$$ |F(x)| = O\left((1-2\beta)^{-d/2}e^{\beta R^2} \right), \; \; \; (1+2\alpha)^{-d/2} e^{-\alpha R^2} \leq G(x) \leq (1-2\beta)^{-d/2} e^{\beta R^2},
$$
It follows from the fact that $\phi_F$ and $\phi_G$ approximate $F$ and $G$ \textit{uniformly} on $B_R$ (and the choice of $R$ and $\epsilon$) that 
$$ \phi_G(x) \geq \frac{1}{2} (1+2\alpha)^{-d/2} e^{-\alpha R^2},
$$
and $\phi_F$ is bounded above by $O\left((1-2\beta)^{-d/2}e^{\beta R^2} \right)$. It follows that
\begin{align*}
     \sup_{\|x\| \leq R} \left| \frac{F(x)}{G(x)} - \frac{\phi_F(x)}{\phi_G(x)} \right| &\lesssim (1+2\alpha)^{d/2} e^{\alpha R^2}  \sup_{\|x|\ \leq R} \left| F(x) - \phi_F(x) \right| \\ &+ (1+2\alpha)^d (1-2\beta)^{-d/2} e^{(2\alpha + \beta)R^2} \sup_{\|x\| \leq R} \left| G(x) - \phi_G(x) \right| \\
     &\lesssim (1+2\alpha)^{d} (1-2\beta)^{-d/2} e^{2(\alpha+\beta)R^2} \left(\epsilon + (1-2\beta)^{-d/2} e^{-\frac{(1-2\beta)R^2}{2}} \right)
\end{align*}
and this concludes the first step of the proof. To approximate $\frac{\phi_F}{\phi_G}$ by a neural network, Lemma \ref{approxquot} states that there exists a neural network $\phi_{quot}: \R^2 \rightarrow \R$ which satisfies
$$ \sup_{x \in [-r,r], y \in [a,b]} \left| \phi_{quot}(x,y) - \frac{x}{y} \right| \lesssim \epsilon,
$$
where the path norm of $\phi_{quot}$ is $O(ba^{-2} M)$, where $M = \max(r^2,b^2a^{-4})$. We apply this lemma with $r = O\left((1-2\beta)^{-d/2}e^{\beta R^2} \right)$, $a = (1+2\alpha)^{-d/2}e^{-\alpha R^2}$ and $b = (1-2\beta)^{-d/2} e^{\beta R^2}$. We conclude that the network $\phi_{F,G}(x) := \phi_{quot}(\phi_F(x), \phi_G(x))$ satisfies
\begin{align*}
    \sup_{\|x\| \leq R} \left| \phi_{F,G}(x) - \frac{F(x)}{G(x)} \right| &\leq \sup_{\|x\| \leq R} \left| \phi_{F,G}(x) - \frac{\phi_F(x)}{\phi_G(x)} \right| + \sup_{\|x\| \leq R} \left| \frac{\phi_F(x)}{\phi_G(x)} - \frac{F(x)}{G(x)} \right| \\
    &\lesssim \epsilon + (1+2\alpha)^{d} (1-2\beta)^{-d/2} e^{2(\alpha+\beta)R^2} \left(\epsilon + (1-2\beta)^{-d/2} e^{-\frac{(1-2\beta)R^2}{2}} \right) \\
    &= O \left(  (1+2\alpha)^{d} (1-2\beta)^{-d/2} e^{2(\alpha+\beta)R^2} \left(\epsilon + (1-2\beta)^{-d/2} e^{-\frac{(1-2\beta)R^2}{2}} \right) \right)
\end{align*}
To conclude, we note that $\|\phi_{quot}\|_{\textrm{path}} = O\left((1-2\beta)^{-3d/2} (1+2\alpha)^{3d} e^{(3\beta + 6 \alpha)R^2} \right)$, and hence by the definition of the path norm,
\begin{align*} \|\phi_{F,G}\|_{\textrm{path}} &= O\left( (1-2\beta)^{-3d/2} (1+2\alpha)^{3d} e^{(3\beta + 6 \alpha)R^2} \cdot \max\left(\|\phi_F\|_{\textrm{path}}, \|\phi_{G}\|_{\textrm{path}} \right) \right) \\ &= O\left((1+2\alpha)^{3d} (1-2\beta)^{-3d/2} e^{6(\beta + \alpha)R^2} \eta(R,\epsilon) \right).
\end{align*}
\end{proof}

We are now in position to prove Proposition \ref{approximationbound}, which controls the approximation error for the function $\frac{F^j_t}{G_t}$ in the $L^2(p_t)$ norm by leveraging the tail decay of $p_t$.
\begin{prop}\label{approximationbound}
    Let $F^j_t$, $G_t$ be defined as in Lemma \ref{score}. Let $\epsilon > 0$ be small enough. Then there exists a neural network $\phi_{F,G,t}^j(x)$ such that, with $(\textbf{s}_t)_j = \mathbb{I}_{\|x\| \geq R} \frac{1}{1-e^{-2t}}\left(-x_j + e^{-t} \phi_{F,G,t}^j \right)$, we have
    $$ \int_{\R^d} \|\textbf{s}_t - \nabla_x \log p_t(x)\|^2 p_t(x) dx = O\left( \max \left( \frac{1}{(1-e^{-2t})^2} (1+2\alpha)^{2d}  \epsilon^{2(1-c(\alpha,\beta))}, \frac{1}{(1-e^{-2t})^3} \epsilon^{1/2} \right) \right),
    $$
    where $c(\alpha,\beta) := \frac{4(\alpha+\beta)}{(1-2\beta)}.$
    In addition, we have $$\|\phi_{t,F,G}^j\|_{\textrm{path}} = O\left((1+2\alpha)^{3d} \epsilon^{-3 c(\alpha,\beta)} \eta(R_0,\epsilon) \right),$$
    where $R_0 = \Theta(\sqrt{d + \log(\epsilon^{-1}}))$.
\end{prop}

\begin{proof}
Let us again write $\phi_{F,G,t}^j = \phi_{F,G}$, $F_t^j = F$ and $G_t = G$ for ease of notation. Let $\phi_{F,G}$ be the neural network constructed in Lemma \ref{approxratio}. Then, for any fixed $R$, it follows from the definition of $\textbf{s}_t$ that
\begin{align*}
    \int_{\R^d} \|\textbf{s}_t - \nabla_x \log p_t(x)\|^2 p_t(x) dx &\leq \frac{1}{(1-e^{-2t})^2} \sum_{j=1}^{d} \sup_{x \in B_R} |\phi_{F,G,t}^j(x) - \frac{F^j_t(x)}{G_t(x)}|^2 \\ &+  \frac{1}{(1-e^{-2t})^2}\int_{\|x\| \geq R} \|\nabla_x \log p_t(x)\|^2 p_t(x) dx = I + II.
\end{align*}
The first term is approximation error on the bounded domain, and we have
$$ I = O \left((1+2\alpha)^{2d} (1-2\beta)^{-d} e^{4(\alpha+\beta)R^2} \left(\epsilon^2 + (1-2\beta)^{-d} e^{-(1-2\beta)R^2} \right) \right)
$$
by Lemma \ref{approxratio}. The second term is truncation error, and we have
\begin{align*}
    II &\leq \E_{p_t}[\|\nabla_x \log p_t(x)\|^4]^{1/2} \cdot \left(\int_{\|x\| \geq R} p_t(x)  dx \right)^{1/2} \\
    &\lesssim \frac{1}{(1-e^{-2t})^2} e^{-\frac{(1-2\beta)R^2}{4}},
\end{align*}
where the bound $\E_{p_t}[\|\nabla_x \log p_t(x)\|^4]^{1/2} \lesssim \frac{1}{(1-e^{-2t})^2}$ follows from Lemma 21 in \citet{chen2023improved} and the bound for the second factor follows from the tail decay of $p_t$ derived in Lemma \ref{subgauss}. To optimize over the cutoff radius $R$, let us choose $R \geq R_0 =  \Omega(\sqrt{d + \log(\epsilon^{-1})})$, so that $(1-2\beta)^{-d} e^{-(1-2\beta)R^2} = \epsilon^2.$ Then it follows that term $I$ satisfies
$$ I =  O \left(\frac{1}{(1-e^{-2t})^2} (1+2\alpha)^{2d}  \epsilon^{2(1-c(\alpha,\beta))} \right)
$$
and $$ II = O\left( \frac{1}{(1-e^{-2t})^3} \epsilon^{1/2} \right).
$$
This gives the bound as stated in the proposition. The path norm follows by setting $R = R_0$ in the path norm bound from Lemma \ref{approxratio}.
\end{proof}

\section{Generalization error of score estimate}\label{gen}
Recall that we are to study the estimation properties of a hypothesis class of the form
$$ \nn_{score}^t(L,K) = \{x \mapsto \frac{1}{1-e^{-2t}}\left(-x + e^{-t} \phi_{NN}(x)\right): \phi_{NN} \in \nn(L,K)\},
$$
where $\nn(L,K)$ is the class of $L$-depth ReLU networks from $\R^d$ to $\R^d$ with path norm bounded by $K$. It follows from the definition of the path norm that functions in $\nn(L,K)$ are $K$-Lipschitz continuous. We also restrict attention to functions in $\nn(L,K)$ which satisfy $|\phi(0)| \leq K$. This mild \footnote{The assumption that $\|\phi(0)\| \leq K$ is mild because the network constructed to approximate the score is uniformly close to the score function around the origin. Therefore, provided $K$ is large enough, the approximating network will satisfy $\|\phi(0)\| \leq K$ anyways.} assumption ensures that the ReLU networks we consider satisfy $\|\phi(x)\| \leq 2K$, which we make frequent use. We will later bound the path norm $K$ in terms of the number of training samples for some concrete examples. We defined the individual loss function at time $t$ by
$$ \ell^t(\phi,x) = \E_{X_t| X_0 = x}\left[\|\phi(t,X_t) - \Psi_t(X_t|X_0)\|^2 \right]
$$
and the associated population risk
$$\mathcal{R}^t(\phi) = \E_{x \sim p_0} [\ell^t(\phi,x)].$$
For $0 < S < R$, define
$$ \ell^t_{R,S}(\phi,x) = \mathbb{I}_{\|x\| \leq S} \E_{X_t| X_0 = x}\left[\|\phi(t,X_t) - \Psi_t(X_t|X_0)\|^2 \cdot \mathbb{I}_{W_R} \right]
$$
and
$$ \ell^t_{S}(\phi,x) = \mathbb{I}_{\|x\| \leq S} \E_{X_t| X_0 = x}\left[\|\phi(t,X_t) - \Psi_t(X_t|X_0)\|^2 \right],
$$
where $W_R$ is the event $\{\sup_{t_0 \leq t \leq T} \|X_t\| \leq R + \|X_0\|\}$. We also define 
$$ \mathcal{R}^t_{R,S}(\phi) = \E_{x \sim p_0} [\ell^t_{R,S}(\phi,x)]
$$
and $\mathcal{R}^t_S(\phi) = \E_{x \sim p_0} [\ell^t_S(\phi,x)]$. In other words, $\mathcal{R}^t_{R,S}$ is the truncated version of the population risk, where the expectation is restricted to the event that the process begins in the ball of radius $S$ and remains in the ball of radius $R+S$ throughout the relevant time interval. We will use the following large deviation bound on the OU process from Theorem A.1 in \citet{oko2023diffusion}.
\begin{lemma}
    Let $X_t$ denote the OU process. Then there is a universal constant $C > 0$ such that for any $0 < S < R$,
    $$ \mathbb{P} \left(\textrm{$\|X_t\| \geq R$ for some $t \in [t_0,T]$}| \|X_0\| \leq S \right) \leq \frac{T}{t_0} e^{-\frac{(R-S)^2}{2Cd}}.
    $$
\end{lemma}
The following regularity bounds on the truncated loss function will be used later to prove a generalization error estimate.
\begin{lemma}\label{boundedliploss}
    Let $\textbf{s}_1(t,x)$, $\textbf{s}_2(t,x) \in \f_{score}$ and write $\textbf{s}_i(t,x) = \frac{1}{1-e^{-2t}}\left(-x+e^{-t} \phi_i(x)\right)$ for $i=1,2$ and $\phi_i \in \f_{NN}$. Then for any $0 < S < R$ and any $x \in \R^d$, we have
    $$ \left| \ell^t_{R,S}(\textbf{s}_1,x) - \ell^t_{R,S}(\textbf{s}_2,x) \right| = \begin{cases}
        O(K(R+S)) \left( \frac{e^{-t}}{1-e^{-2t}} \right)^2 \E_{X_t|X_0 = x} \|\phi_1(x) - \phi_2(x)\|, \; x \in B_S \\
        0, \; \|x\| > S.
    \end{cases} 
    $$
    In addition, the truncated loss function is bounded: for any $x \in \R^d$ and any $\phi \in \f_{NN}$, we have
    $$ \ell_{R,S}^t(\textbf{s},x) = O \left( \left( \frac{e^{-t}}{1-e^{-2t}} \right)^2 K^2(R+S)^2 \right).
    $$
    where $\textbf{s}(t,x) = \frac{1}{1-e^{-2t}}(-x+e^{-t} \phi(x))$.
\end{lemma}
\begin{proof}
    Using the definition of $\ell^t_{R,S}$, we have, for any $x \in \R^d$,
    \begin{align*}
        &\left| \ell^t_{R,S}(\textbf{s}_1,x) - \ell^t_{R,S}(\textbf{s}_2,x) \right| \\
        &= \left| \mathbb{I}_{B_S}(x) \left( \frac{e^{-t}}{1-e^{-2t}} \right)^2 \E_{X_t|X_0 = x, W_R} [\|\phi_1(X_t) - X_0\|^2 - \|\phi_2(X_t) - X_0\|^2] \right| \\
        &\leq 2(2K(R+S)+S) \mathbb{I}_{B_S}(x) \left( \frac{e^{-t}}{1-e^{-2t}} \right)^2 \E_{X_t|X_0 = x, W_R} \|\phi_1(X_t) - \phi_2(X_t)\| \\
    \end{align*}
    Note that for the first inequality, we have used that the map $x \mapsto \|x\|^2$ is $2R$ Lipschitz on $B_R$, and that $\sup_{i=1,2, \; t_0 \leq t \leq T} \|\phi_i(X_t) - X_0\| \leq 2K(R+S) + S$ under the given assumptions. The proof of boundedness follows similarly: for $\phi \in \f_{NN}$ and $\textbf{s} = \frac{1}{1-e^{-2t}}(-x + e^{-t}\phi(x))$, we have by the Cauchy-Schwarz inequality that
    \begin{align*}
        \left|  \ell^t_{R,S}(\textbf{s},x) \right| &= \left|\mathbb{I}_{B_S}(x) \left( \frac{e^{-t}}{1-e^{-2t}} \right)^2 \E_{X_t|X_0 = x} [\|\phi(X_t) - X_0\|^2 \cdot \mathbb{I}_{W_R} ] \right| \\
        &\lesssim \mathbb{I}_{B_S}(x) \left( \frac{e^{-t}}{1-e^{-2t}} \right)^2 \left(K^2(R+S)^2 + S^2 \right) = O \left( \left( \frac{e^{-t}}{1-e^{-2t}} \right)^2 K^2(R+S)^2 \right).
    \end{align*}
\end{proof}
Before proving the main generalization error bound, we need to control the truncation error incurred from using $\mathcal{R}^t_{R,S}$ in place of $\mathcal{R}^t$.
\begin{prop}\label{risktruncationerror}
    For any $r_f < S < R$ and any $\textbf{s} \in \nn^{\textrm{score},t}(L,K)$, we have
    $$ \left| \mathcal{R}^t_{R,S}(\textbf{s}) - \mathcal{R}^t(\textbf{s}) \right| = O \left(  \left(\frac{e^{-t}}{1-e^{-2t}} \right)^2 K^2 \left( (T/t_0)^{1/2} e^{-\frac{(R-S)^2}{4Cd}} + e^{-\frac{(1-2\beta)R^2}{2}} \right)\right).
    $$
\end{prop}
\begin{proof}
     Let $\textbf{s}(x) = \frac{1}{1-e^{-2t}} \left( -x + e^{-t} \phi(x) \right)$ with $\phi \in \nn(L,K)$. We have
     $$ \mathcal{R}^t_{R,S}(\textbf{s}) - \mathcal{R}^t(\textbf{s}) = \left(\mathcal{R}^t_{R,S}(\textbf{s}) - \mathcal{R}_R^t(\textbf{s}) \right) + (\mathcal{R}^t_{R}(\textbf{s}) - \mathcal{R}^t(\textbf{s})) = I + II.
     $$
     For the first term, we have
     \begin{align*}
         I &= \left( \frac{e^{-t}}{1-e^{-2t}} \right)^2 \E_{x \sim p_0} \left[ \mathbb{I}_{B_S}(x) \cdot \E_{X_t| X_0 = x} \left[ \|X_0-\phi(X_t)\|^2 \cdot \mathbb{I}_{W_R^c} \right] \right] \\
         &\leq \left( \frac{e^{-t}}{1-e^{-2t}} \right)^2 \E_{x \sim p_0} \left[ \mathbb{I}_{B_S}(x) \cdot \mathbb{P}(W_R^c)^{1/2} \cdot \E\left[ \|X_0 - \phi(X_t)\|^4 \right]^{1/2} \right] \\
         &\lesssim \left( \frac{e^{-t}}{1-e^{-2t}} \right)^2 \cdot \left( T/t_0 \right)^{1/2} e^{-\frac{(R-S)^2}{4Cd}} \cdot \E_{x \sim p_0} \left[\mathbb{I}_{B_S}(x) \E_{X_t|X_0=x} \left[ \|X_0 - \phi(X_t)\|^4 \right]^{1/2} \right].
     \end{align*}
     To bound the last term, we have
     \begin{align*}
         \E_{x \sim p_0} \left[\mathbb{I}_{B_S}(x) \E_{X_t|X_0=x} \left[ \|X_0 - \phi(X_t)\|^4 \right]^{1/2} \right] &\lesssim \E_{x \sim p_0} \left[ \mathbb{I}_{B_S}(x) \E_{X_t|X_0=x}[\|X_0\|^4 + \|\phi(X_t)\|^4]^{1/2} \right] \\
         &\leq \left(\E_{x \sim p_0}\left[\mathbb{I}_{B_S}(x) \left( \|x\|^4 + \E_{X_t|X_0=x}[\|\phi(X_t)\|^4] \right] \right) \right)^{1/2} \\
         &\leq \left( \E_{x \sim p_0}[\|x\|^4 \cdot \mathbb{I}_{B_S}(x)] + \E_{x \sim p_t}[\|\phi(x)\|^4] \right)^{1/2} \\
         &\lesssim \left(S^4 + K^4\E_{x \sim p_t}[\|x\|^4] \right)^{1/2} \\
         &= O(K^2).
     \end{align*}
     This proves that
     $$ I = O \left( \left( \frac{e^{-t}}{1-e^{-2t}} \right)^2 \cdot \left( T/t_0 \right)^{1/2} e^{-\frac{(R-S)^2}{4Cd}} K^2 \right).
     $$
     For term $II$, we have
     \begin{align*}
         II &= \left(\frac{e^{-t}}{1-e^{-2t}} \right)^2 \int_{\|x\| \geq S} \E_{X_t|X_0=x}[\|X_0-\phi(X_t)\|^2] p_0(x) dx \\
         &\lesssim \left(\frac{e^{-t}}{1-e^{-2t}} \right)^2 \left( \int_{\|x\| \geq S} \|x\|^2 p_0(x) dx + \int_{\|x\| \geq S} K^2\|x\|^2 p_t(x) dx \right) \\
         &\lesssim  \left(\frac{e^{-t}}{1-e^{-2t}} \right)^2 (2\pi (1-2\beta))^{-d/2} \left( \int_{\|x\| \geq S} \|x\|^2 e^{-\frac{(1-2\beta)\|x\|^2}{2}} + K^2 \int_{\|x\| \geq S} \|x\|^2 e^{-\frac{(1-2\beta)\|x\|^2}{2}} dx \right) \\
         &\lesssim  \left(\frac{e^{-t}}{1-e^{-2t}} \right)^2 K^2 e^{-(1-2\beta)\frac{S^2}{2}}.
     \end{align*}
     where we have used the uniform-in-time sub-Gaussian upper bound on the process density from Lemma \ref{subgauss}. Combining the bounds for terms $I$ and $II$ gives the bound as stated in the lemma.
\end{proof}
The following lemma controls the Rademacher complexity of our hypothesis class.
\begin{lemma}\label{radcomplexbound}
    For $t > 0$, $R > 0$, let $\nn^{\textrm{score},t}(L,K)$ be as defined in Section \ref{gen} and let $\mathbb{S} = \{x_1, \dots, x_N\}$ be a collection of points in $\R^d$. Then
    \begin{align*} \textrm{Rad}_N\left( \ell_{R,S} \circ \nn^{\textrm{score},t}(L,K), \mathbb{S} \right) &:= \E_{\epsilon_i \sim \textrm{Ber}\left(\{\pm1\} \right)} \left[ \sup_{\textbf{s} \in \nn^{\textrm{score},t}(L,K)} \frac{1}{N} \sum_{i=1}^{N} \epsilon_i \cdot \left( \ell_{R,S} \circ \textbf{s}\right) (x_i) \right] \\ &\lesssim 2^L d K^2 (R+S)^2 \left( \frac{e^{-t}}{1-e^{-2t}} \right)^2 \cdot \frac{1}{\sqrt{N}}.
    \end{align*}
\end{lemma}
\begin{proof}
    We can assume that $\mathbb{S} \subseteq B_{S}$, because $\ell_{R,S}(x,\textbf{s}) = 0$ for any $x \notin B_S$ and any $\textbf{s}$. By Lemma \ref{boundedliploss} and Lemma \ref{vectorcontract} (a vector version of the contraction inequality for Rademacher complexity), it holds that
    $$ \textrm{Rad}_N\left( \ell_{R,S} \circ \nn^{\textrm{score},t}(L,K), \mathbb{S} \right) \leq \sqrt{2}d K (R+S) \left( \frac{e^{-t}}{1-e^{-2t}} \right)^2 \cdot \textrm{Rad}_N\left( \nn^{\textrm{score},t}(L,K), \mathbb{S} \right).
    $$
    Then, since $\nn^{\textrm{score},t}(L,K) = \{x \mapsto \frac{1}{1-e^{-2t}}\left( -x + e^{-t} \phi(x)\right): \phi \in \nn(L,K)\}$, it holds by the scaling and translation properties of Rademacher complexity that
    $$ \textrm{Rad}_N\left( \nn^{\textrm{score},t}(L,K), \mathbb{S} \right) \leq \E_{X_t^i |X_0^i = x_i, W_R} \frac{e^{-t}}{1-e^{-2t}} \cdot \textrm{Rad}_N\left( \nn(L,K), \mathbb{S^t} \right),
    $$
    where $\mathbb{S}^t = \{X_t^1, \dots, X_t^N \}$ are now random variables. It is well-known(e.g.,  Lemma 3.13 in \citet{wojtowytsch2020banach}) that
    $$ \textrm{Rad}_N\left( \nn(L,K), \mathbb{S}^t \right) \leq \max_i \|X_i^t\|_{\infty} \cdot 2^{L} K \cdot \sqrt{\frac{2 \log(2d+2)}{N}}.
    $$
    Note that on the event $W_R$, we have $\max_i \|X_t^i\|_{\infty} \leq (R+S)$. Putting everything together gives the desired bound.
\end{proof}

The following lemma bounds the error between $p_0$ and $p_t$, the forward process at time $t$, by bounding the derivative of the the function $t \mapsto KL(p_t \parallel p_0)$. We emphasize that the estimate is only useful for short times.

\begin{lemma}\label{KLderivative}
    Define $M_{\beta}(f) := \int_{\R^d} \left \| \nabla f\left ( \frac{x}{1-2\beta} \right) \right\|^2 \gamma_d(dx).$ For any $t > 0$, we have $D_{KL}(p_t \parallel p_0) \lesssim M_{\beta}(f) t$.
\end{lemma}
\begin{proof}
    It suffices to prove that the time derivative of the KL divergence satisfies \begin{equation}\label{KLderivative2}
        \partial_t D_{KL}(p_t \parallel p_0) \lesssim M_{\beta}(f).
    \end{equation}
    To prove inequality \ref{KLderivative2}, we differentiate the relative entropy:
    \begin{align*}
        \partial_t D_{KL}(p_t \parallel p_0) &= \partial_t \int p_t(x) \log\left(\frac{p_t}{p_0}(x)\right) dx \\
        &= \int \partial_t \left( p_t \right)(x) \log\left(\frac{p_t}{p_0}(x)\right) dx + \int p_t \partial_t \left( \log\left(\frac{p_t}{p_0}(x) \right) \right) dx.
    \end{align*}
    But the second term is equal to zero, because
    \begin{align*}
         \int p_t \partial_t \left( \log\left(\frac{p_t}{p_0} \right) \right) dx &= \int p_t \cdot \frac{\partial_t p_t}{p_t} dx \\
         &= \int \partial_t p_t dx \\
         &= \partial_t \left( \int p_t dx \right) = 0,
    \end{align*}
    where the last line follows because $p_t$ is a probability density function for all $t$, and hence $\int p_t dx = 1$ for all $t$. Now, recall that $p_t$ satisfies the \textit{Fokker-Planck} equation $\partial_t p_t(x) = \nabla \cdot (x p_t(x)) + \Delta p_t(x)$ for $t > 0$. This means that, with $\gamma_d(x)$ as the standard Gaussian density,
    \begin{align*}
        \partial_t D_{KL}(p_t \parallel p_0) &=  \int \left( \nabla \cdot (p_t(x) x) + \Delta p_t(x) \right) \log\left(\frac{p_t}{p_0}(x)\right) dx \\
        &= \int \nabla \cdot \left(\nabla \log \left( \frac{p_t}{\gamma_d}(x) \right) p_t(x) \right) \log \left( \frac{p_t}{p_0}(x) \right) dx.
    \end{align*}
    Integrating by parts, we have
    \begin{align*}
         \partial_t D_{KL}(p_t \parallel p_0) &= \int \nabla \cdot \left(\nabla \log \left( \frac{p_t}{\gamma_d}(x) \right) p_t(x) \right) \log \left( \frac{p_t}{p_0}(x) \right) dx \\ &= - \int \nabla \log \left( \frac{p_t}{\gamma_d}(x) \right) \cdot \nabla \log \left( \frac{p_t}{p_0}(x) \right) p_t(x) dx,
    \end{align*}
    since the Gaussian tail decay of $p_t$ ensures that the boundary term vanishes. Now, note from the Cauchy-Schwarz inequality that 
    \begin{align*}
        &- \int \nabla \log \left( \frac{p_t}{\gamma_d}(x) \right) \cdot \nabla \log \left( \frac{p_t}{p_0}(x) \right) p_t(x) dx \\
        &= - \int \left( \left( \nabla \log \left( \frac{p_t}{p_0}(x) \right) + \nabla \log \left( \frac{p_0}{\gamma_d}(x) \right) \right) \cdot \nabla \log \left( \frac{p_t}{p_0}(x) \right) \right) p_t(x) dx \\
        &\leq -I(p_t \parallel p_0) + \sqrt{I(p_t \parallel p_0)} \cdot \left( \int \left \|\nabla \log \left( \frac{p_0}{\gamma_d}(x) \right) \right\|^2 p_t(x) dx \right)^{1/2},
    \end{align*}
    where 
    \begin{align*}
        I(p_t \parallel p_0) := \int \left\| \nabla \log \left( \frac{p_t}{p_0}(x) \right) \right\|^2 p_t(x) dx
    \end{align*}
    is the relative Fisher information of $p_t$ with respect to $p_0$. Using the inequality $ax - a^2 \leq \frac{1}{4}x^2$ for $a > 0$, it follows that
    \begin{align*}
        \partial_t D_{KL}(p_t \parallel p_0) &\leq \frac{1}{4}\int \left \|\nabla \log \left( \frac{p_0}{\gamma_d}(x) \right) \right\|^2 p_t(x) dx.
    \end{align*}
    But recall that $p_0(x) = \frac{1}{Z} e^{-\|x\|^2/2 +f(x)}$. Therefore 
    \begin{align*}
        \int \left \|\nabla \log \left( \frac{p_0}{\gamma_d}(x) \right) \right\|^2 p_t(x) dx &= \int \left\| \nabla f(x) \right\|^2 p_t(x) dx \\
        &\lesssim \int \left\|\nabla f\left(\frac{x}{1-2\beta} \right) \right\| \gamma_d(dx) \\
        &:= M_{\beta}(f).
    \end{align*}
    This proves the inequality \ref{KLderivative} and hence the original claim as well.
\end{proof}

We are now in position to prove the main generalization bound. The following result is the same content as Proposition \ref{informalgeneralization}, but stated more precisely.
\begin{prop}\label{generalizationmain}
    Let $\mathcal{R}^t$ denote the population risk functional at time $t$, let $\widehat{\mathcal{R}}^t$ denote the associated empirical risk functional, and let $\mathcal{R}_{R}^t = \mathcal{R}_{2R,R}^t$ and $\widehat{\mathcal{R}}_R^t = \widehat{\mathcal{R}}_{2R,R}^t$ denote the truncated risk functionals as defined in Section \ref{gen}. Let $\widehat{\textbf{s}}$ denote the minimizer of $\widehat{\mathcal{R}}^t_R$ over the neural network class $\mathcal{NN}^{score,t}(L,K)$. Then, with $L = L_f + 5$ and $K = O\left((1+2\alpha)^{3d} \epsilon^{-3 c(\alpha,\beta)} \eta(\sqrt{d + \log(\epsilon^{-1})},\epsilon) \right)$, we have
    $$ \mathcal{R}^t(\widehat{\textbf{s}}) = O((1+2\alpha)^{2d} \epsilon^{2(1-c(\alpha,\beta))}),
    $$
    with probability $1-\textrm{poly}(1/N)$, provided the number of training samples is 
    $$ N = \Omega \left(2^{2 L_f+10} d^2 t_0^{-6} \epsilon^{-4-8c(\alpha,\beta)} \eta^{4}\left( R_0, \epsilon \right)  \right).
    $$
\end{prop}
We note that, evidenced by Lemma \ref{risktruncationerror}, $\widehat{\mathcal{R}}^t = \widehat{\mathcal{R}}^t_R$ with high probability.
\begin{proof}
    Let $\textbf{s}^{\ast} = \textrm{argmin}_{\textbf{s} \in \mathcal{NN}^{score,t}(L,K)} \mathcal{R}^t(\textbf{s})$ and $\textbf{s}_R^{\ast} = \textrm{argmin}_{\textbf{s} \in \mathcal{NN}^{score,t}(L,K)} \mathcal{R}_R^t(\textbf{s}).$ Then we have
    \begin{align*}
        \mathcal{R}^t(\widehat{\textbf{s}}) = \left( \mathcal{R}^t(\widehat{\textbf{s}})  - \mathcal{R}_R^t(\widehat{\textbf{s}}) \right) + \left( \mathcal{R}^t_R(\widehat{\textbf{s}}) - \mathcal{R}^t_R(\textbf{s}_R^{\ast} ) \right) + \left( \mathcal{R}^t_R(\textbf{s}_R^{\ast}) - \mathcal{R}^t(\textbf{s}^{\ast}) \right) + \mathcal{R}^t(\textbf{s}^{\ast}) = I + II + III + IV.
    \end{align*}
    Term $I$ is the truncation error; by Lemma \ref{risktruncationerror}, we have
    $$ I = O \left( \left( \frac{e^{-t}}{1-e^{-2t}} \right)^2 K^2 (T/t_0)^{1/2} e^{-\frac{R^2}{4Cd}} + e^{\frac{-(1-2\beta)R^2}{2}} \right).
    $$
    Term $II$ is the generalization error for the truncated risk; by Lemmas \ref{radcomplexbound} and \ref{boundedliploss} (controlling the Rademacher complexity of the relevant function class and uniformly bounding the loss $\ell_R$) and Theorem 26.4 in \citet{shalev2014understanding}, we have
    $$ II = O \left( \left( \frac{e^{-t}}{1-e^{-2t}} \right)^3 \cdot 2^L d K^2 R^2 \cdot \sqrt{\frac{1}{N}} + \left(\frac{e^{-t}}{1-e^{-2t}} \right)^2 \cdot K^2 R^2 \cdot \sqrt{\frac{2\log(2/\delta)}{N}}  \right),
    $$
    with probability at least $1-\delta$. Term $III$ is nonpositive, because $\mathcal{R}_R^t(\cdot) \leq \mathcal{R}^t(\cdot)$ for all $R, t > 0$, and hence $\inf \mathcal{R}_R^t(\cdot) \leq \inf \mathcal{R}^t(\cdot)$. Term $IV$ is the approximation error; by Theorem \ref{approximationbound}, it is $O \left( (1+2\alpha)^{2d} \epsilon^{2(1-c(\alpha,\beta))} \right)$, where we recall $c(\alpha,\beta) = \frac{4(\alpha+\beta)}{1-2\beta}$, provided that $K = O\left((1+2\alpha)^{3d} \epsilon^{-3 c(\alpha,\beta)} \eta(R_0,\epsilon) \right)$ for $R_0 = \Theta\left(\sqrt{d + \log(\epsilon^{-1}}\right)$ and $L = L_f + 5$ (recall that $\eta(R,\epsilon)$ is the path norm of the network needed to approximate $f$ to accuracy $R\epsilon$ uniformly over $B_R$).

    To balance terms $I$, $II$ and $IV$, let us first choose $R$ large enough that term $I$ is the same order as the approximation error term $IV$. Due to the exponential decay in $R$ of term $I$, this holds for $R$ only logarithmic in all relevant parameters. Let us also take $\delta$ to be polynomial in $1/N$. It then suffices to balance $N$ and $\epsilon$ so that terms $II$ and $IV$ are of the same order, and up to logarithmic factors this amounts to solving
    $$ \left( \frac{e^{-t}}{1-e^{-2t}} \right)^3 \cdot 2^{L_f+5} d K^2 \cdot \sqrt{\frac{1}{N}} = \Theta \left(\max \left( \frac{1}{(1-e^{-2t})^2}(1+2\alpha)^{2d} \epsilon^{2(1-c(\alpha,\beta))}, \frac{1}{(1-e^{-2t})^3} \epsilon^{1/2}  \right)\right).
    $$ 
    Since $\left( \frac{e^{-t}}{1-e^{-2t}} \right) = O(t^{-1})$ and $K = O\left((1+2\alpha)^{3d} \epsilon^{-3 c(\alpha,\beta)} \eta(R_0,\epsilon) \right),$ it therefore holds that if we have \begin{align*}N &= \Omega \Bigg( \max \Bigg(2^{2 L_f+10} d^2 t_0^{-6} (1+2\alpha)^{12d} \epsilon^{-4} \eta^{4}\left( R_{\epsilon}, (1+2\alpha)^{2d}\epsilon^{2(1-2c(\alpha,\beta))} \right) , \\ &2^{2L_f+10} (1+2\alpha)^{12d} t^{-6-72c(\alpha,\beta)} \epsilon^{-4-48c(\alpha,\beta)} \eta^4(\tilde{R}_{\epsilon},t^6 \epsilon^4) \Bigg) \end{align*} samples (where $R_{\epsilon} = \sqrt{d+\frac{1}{1-c}\log \left( \frac{(1+2\alpha)^d}{\epsilon t^2}\right)}$ and $\tilde{R}_{\epsilon} = \sqrt{d-\log(t^6 \epsilon^4)}$), our score estimation error is $O(\epsilon^2).$
\end{proof}

We now employ existing sampling guarantees to prove that the distribution returned by the SGM is close to the true data distribution.
\begin{proof}[Proof of Proposition \ref{distributionestimationmain}]
By Proposition \ref{generalizationmain} (which controls the score estimation error) and Theorem 1 in \citet{benton2023linear} (which controls the sampling error of SGMs in terms of the score estimation error) we have that
$$ KL(p_{t_0} \parallel \widehat{p}) \lesssim  (1+2\alpha)^{2d} \epsilon^{2(1-c(\alpha,\beta))} + \kappa^2 d M + \kappa d T + de^{-2T},
$$
where $[t_0,T]$ is the time interval of the forward process, $\kappa$ is the maximum step size for the exponential integrator, and $M$ is the number of iterations of the exponential integrator. Choosing $M,$ $\kappa,$ and $T$ to scale with $\epsilon$ as described in the statement of Prop \ref{distributionestimationmain} yields that each term is of order at most $(1+2\alpha)^{2d} \epsilon^{2(1-c(\alpha,\beta))}$, and hence $TV(p_{t_0},\widehat{p}) = O((1+2\alpha)^{d} \epsilon^{1-c(\alpha,\beta)})$. It now remains to bound $TV(p_{t_0},p_0)$, and Lemma \ref{KLderivative} shows that
$$ TV(p_{t_0},p_0) \lesssim M_{\beta}(f) t_0^{1/2}.
$$
It follows that if we scale $t_0$ so that the above expression is $O((1+2\alpha)^{2d} \epsilon^{2(1-c(\alpha,\beta))})$, then we get $TV(\widehat{p},p_0) \leq TV(\widehat{p},p_{t_0}) + TV(p_{t_0}, p_0) \lesssim (1+2\alpha)^d \epsilon^{1-c(\alpha,\beta)}.$ The number of samples $N$ required to achieve this error was derived in the proof of Proposition \ref{generalizationmain}.
    
\end{proof}

\section{Distribution estimation and proofs for concrete examples}\label{appendixdistributionestimation}
We give the proofs of the distribution estimation bounds for the general case and for the concrete examples discussed. These proofs are quite short and follow easily from the score estimation results.

\begin{proof}[Proof of Theorem \ref{distributionestimationmain}]
By Proposition \ref{informalgeneralization}, we know that with $N \geq N_{\epsilon,t}$ samples, our score estimator $s_{t}$ is $O(\epsilon)$-close to the true score at time $t$ in $L^2(p_t)$. By Theorem 1 in \cite{benton2023linear}, the exponential integrator scheme with parameters as defined in Theorem \ref{distributionestimationmain} produces a distribution $\widehat{p}$ which satisfies $TV(\widehat{p},p_{t_0}) = O(\epsilon).$ All that remains then is to bound $TV(p_0,p_{t_0})$. By Lemma \ref{KLderivative}, we have under Assumption \ref{datadistr2} that $TV(p_{t_0},p_0) \lesssim M_{\beta} t_0$. It follows that choosing $t_0$ as defined in Theorem \ref{distributionestimationmain} balances $TV(\widehat{p},p_{t_0})$ and $TV(p_{t_0},p_0)$, so that $TV(\widehat{p},p_0) = O(\epsilon).$
\end{proof}

\begin{proof}[Proof of Theorem \ref{barronestimation}]
We note that a Barron function $f$ grows linearly with a constant $c_f \leq \|f\|_{\mathcal{B}}$, and therefore, for any $R > 0$, we have for all $\|x\| \geq R$ that $|f(x)| \leq c_f \|x\| = c_f \frac{\|x\|}{\|x\|^2} \leq \frac{c_f}{R} \|x\|^2$. This shows that $f$ satisfies that quadratic growth/decay condition of Assumption \ref{datadistr2} with constants $\alpha = \beta = \frac{c_f}{R}$ We choose $\tilde{R}_{\epsilon}= \Omega(\sqrt{d + \log(t^6 \epsilon^4)})$ to be the optimal cutoff radius for the approximation error argument. In addition, we know \citep{ma2020towards, klusowski2018approximation} that for Barron $f$, there exists a shallow ReLU neural network $\phi$ such that 
$$ \sup_{\|x\| \leq R} |f(x) - \phi(x)| \leq R \epsilon
$$
and $\|\phi\|_{\textrm{path}} \lesssim \|f\|_{\mathcal{B}}$. The result essentially follows from the estimates in Theorem \ref{distributionestimationmain} by replacing $\alpha$ and $\beta$ with $\frac{c_f}{\tilde{R}_{\epsilon}}$ and replacing $\eta(R,\epsilon)$ with $\|f\|_{\mathcal{B}}$; note that in this case $c(\alpha,\beta) \leq \frac{8c_f}{\tilde{R}_{\epsilon}} \leq \delta$ by assumption.
\end{proof}

For the Gaussian mixture example, we first need to show that the log-likelihood has the local approximation property, which is the content of the following Lemma. The approximating network has two hidden layers; the inner layer approximates the density and the outer layer approximates $\log(x)$ on the image of the density.

\begin{lemma}\label{approxmixture}
    Suppose that $p_0(x) = \frac{1}{2} \left(\frac{1}{Z_1} e^{-\frac{\|x-x_1\|^2}{2\sigma_{\min}^2}} + \frac{1}{Z_2} e^{-\frac{\|x-x_2\|^2}{2\sigma_{\max}^2}} \right).$ Then for any $R > 0$, there exists a ReLU network $f_{NN}$ with two hidden layers such that
    $$ \sup_{\|x\| \leq R} |f(x) - f_{NN}(x)| \leq \epsilon.
    $$
    Moreover, $f_{NN}$ satisfies $\|f\|_{\textrm{path}} = O \left(d \cdot \sigma_{\min}^{-1} \cdot \max_{\|x\| \leq R} p_0^{-1}(x) \right)$.
\end{lemma}
\begin{proof}
    By the theory of spectral Barron functions (e.g., \citet{barron1993universal} and \citet{klusowski2018approximation}) there exists a shallow ReLU network $f_{mix}$ such that $\sup_{\|x\| \leq R} |f_{mix}(x) - p_0(x)| \leq \epsilon$ and $\|f_{mix}\|_{\textrm{path}} = O(d \sigma_{\min}^{-1}).$ We also know from \citet{wojtowytsch2020representation} that $x \mapsto \log(x)$ can be locally approximated to error $\epsilon$ on $[a,b]$ by a network $f_{\log}$ with $\|f_{\log}\|_{\textrm{path}} = O(a^{-1})$. We set $a = \min_{\|x\| \leq R} p_0(x)$ and $b = \max_{\|x\| \leq R} p_0(x)$. It is then clear that the network $\tilde{f}_{NN} = f_{\log} \circ f_{mix}$ satisfies $\sup_{\|x\| \leq R} |\log(p_0(x)) - \tilde{f}_{NN}(x)| \leq \epsilon$ and $\|\tilde{f}_{NN}\|_{\textrm{path}} = O\left( d \cdot \sigma_{\min}^{-1} \cdot )\max_{\|x\| \leq R} p_0^{-1}(x) \right).$ To conclude, we note that by Lemma \ref{approxbuildingblocks}, the map $x \mapsto \|x\|^2/2$ can be approximated on $\{\|x\| \leq R\}$ by a network $f_{norm}$ to accuracy $\epsilon$, and $f_{norm}$ can be taken to satisfy $\|f_{norm}\|_{\textrm{path}} = O(dR^2)$. In turn, the network $f_{NN} = \tilde{f_{NN}} + f_{norm}$ satisfies
    $$ \sup_{\|x\| \leq R} |f(x) - f_{NN}(x)| \leq \epsilon
    $$
    and $\|f_{NN}\|_{\textrm{path}} =  O\left( d \cdot \sigma_{\min}^{-1} \cdot )\max_{\|x\| \leq R} p_0^{-1}(x) \right).$
\end{proof}

We are now equipped to give a simple proof of the distribution estimation result for Gaussian mixtures.

\begin{proof}[Proof Of Proposition \ref{gaussmixtureestimation}]
Suppose $$p_0(x) = \frac{1}{2} \left(\frac{1}{Z_1} e^{-\frac{\|x-x_1\|^2}{2\sigma_{\min}^2}} + \frac{1}{Z_2} e^{-\frac{\|x-x_2\|^2}{2\sigma_{\max}^2}} \right)$$ is a mixture of two Gaussians. Fix some $0 < \delta \ll 1$. Then there exists an $r_0(\delta) > 0$ (depending on $x_1, x_2, \sigma_{\min}, \sigma_{\max}$) such that
\begin{equation}\label{quadraticboundgaussianmixture} -\left( \frac{1+\delta}{2\sigma_{\min}^2} \right)\|x\|^2 \leq \log p_0(x) \leq -\left( \frac{1-\delta}{2\sigma_{\max}^2} \right) \|x\|^2, \; \; \; \forall \; \|x\| > r_0(\delta)
\end{equation}
and therefore we can write $p_0(x) = \frac{1}{Z} e^{-\|x\|^2/2 + f(x)}$, where $f(x)$ satisfies
$$ -\alpha \|x\|^2 := - \left( \frac{1+\delta-\sigma_{\min}^2}{2\sigma_{\min}^2} \right) \|x\|^2 \leq f(x) \leq \beta\|x\|^2 := \left( \frac{\sigma_{\max}^2 + \delta - 1}{2\sigma_{\max}^2} \right)\|x\|^2
$$
whenever $\|x| > r_0$. For $\delta$ small enough, the assumption the $c(\alpha,\beta) < 1$ holds, for instance, whenever $\sigma_{\min}^2 > \frac{2}{3}$ and $\sigma_{\max}^2 \leq 1$. If $\epsilon > 0$ is small enough that $r_0(\delta) \leq \min\left(R_{\epsilon},\tilde{R}_{\epsilon} \right),$ then the result then follows from Theorem \ref{distributionestimationmain} by using the above values of $\alpha$ and $\beta$ and using the complexity measure determined in Lemma \ref{approxmixture} in place of $\eta(R,\epsilon).$ In the special case that $\sigma_{\min}^2 = \sigma_{\max}^2 = 1$, we have $\alpha = \beta = \delta/2$ and $c(\alpha,\beta) \leq 4\delta$. In addition, in this case we also have $\sup_{x \leq \tilde{R}_{\epsilon}} p_0^{-4}(x) \lesssim e^{d/2} t^{-3} \epsilon^{-2}.$
\end{proof}

\section{Background on neural networks}
\subsection{Path norms}
Recall that a shallow ReLU neural network is a function $\phi: \R^d \rightarrow \R^k$ whose $j^{\textrm{th}}$ component is given by
$$ (\phi)_j(x) = \sum_{i=1}^{m} a_{ij} (w_i^T x+b_i)^{(+)}, \; \; w_i \in \R^d, \; a_{ij}, b_i \in \R, \; 1 \leq j \leq k,
$$
and a deep ReLU neural network is a composition of shallow ReLU networks. For scalar valued networks, we define the \textit{path norm} by
$$ \|\phi\|_{\textrm{path}} = \inf \sum_{i=1}^{m} |a_i| \left(\|w_i\|_1 + |b_i| \right),
$$
where the infimum is taken over all choices of parameters $(a_i,w_i,b_i)$ such that $\phi(x) = \sum_{i=1}^{m} a_i (w_i^T x + b_i)^{(+)}$. We extend the path norm to vector-valued shallow networks by
$$ \|\phi\|_{\textrm{path}} = \max_{1 \leq j \leq k} \|(\phi)_j\|_{\textrm{path}}, \; \phi: \R^d \rightarrow \R^k
$$
and to deep networks by
$$ \|\phi\|_{\textrm{path}} = \inf_{\phi_1, \dots \phi_L} \|\phi_1\|_{\textrm{path}} \cdot \dots \cdot \|\phi_L\|_{\textrm{path}},
$$
where the infimum is over all representations of $\phi$ as a composition of shallow networks. A more thorough study of path norms can be found in \citet{wojtowytsch2020banach}. 

The path norm captures how large the weights are in an average (i.e., $\ell^1$) sense. Intuitively, a network with a large path norm is not likely to generalize well to unseen data, because its pointwise values depend on large cancellations. In contrast, networks with small path norm provably generalize well, in the sense of Rademacher complexity. The following result due to \citet{wojtowytsch2020banach} makes this precise.
\begin{prop}
    Let $\mathcal{NN}_{L,K}$ denote the set of $L$-layer ReLU networks whose path norm is bounded by $K$. Let $S = \{x_1, \dots\ x_N\}$ denote a set of points in $\R^d$. Then the empirical Rademacher complexity of $\mathcal{NN}_{L,K}$ is bounded by
    $$ \textrm{Rad}(\mathcal{NN}_{L,K}, S) := \E_{\epsilon_i \sim \textrm{Ber}(\{\pm 1\}) } \sup_{f \in \mathcal{NN}_{L,K}}
    \frac{1}{N} \sum_{i=1}^{N} \epsilon_i f(x_i) \leq \max_{i} \|x_i\|_{\infty} \cdot 2^L \sqrt{\frac{2 \log(2d+2)}{N}}.
    $$
\end{prop}

\subsection{Approximation of helper functions}
We will need to approximate some simple functions by shallow ReLU neural networks; the next result shows that we can do this efficiently (in the sense of path norms). We emphasize that these results are known (e.g., in \citet{wojtowytsch2020representation}), but we provide the full proofs for the sake of completeness.
\begin{lemma}\label{approxbuildingblocks}
    Let $\epsilon > 0$, $R > 0$, and $-\infty < a < b < \infty$.
    \begin{enumerate}
        \item There exists a shallow ReLU neural network $\phi_{exp}: \R \rightarrow \R$ with $O(\epsilon^{-2} e^{2b})$ neurons such that $\sup_{x \in [a,b]} \left| \phi_{exp}(x) - e^x\right| = O(\epsilon).$ In addition, $\phi_{exp}$ satisfies $\|\phi_{exp}\|_{\textrm{path}} = O(e^b).$
        \item There exists a shallow ReLU neural network $\phi_{prod}: \R^2 \rightarrow \R$ with $O(\epsilon^{-2} R^4)$ neurons such that $$\sup_{(x,y) \in [-R,R]^2} |\phi_{prod}(x,y) - xy|  = O(\epsilon).$$ In addition, $\phi_{prod}$ satisfies $\|\phi_{prod}\|_{\textrm{path}} = O(R^2)$.
        \item If $a > 0$, then exists a shallow neural $\phi_{inv}: \R \rightarrow \R$ with $O(\epsilon^{-2} b^2 a^{-4})$ parameters, such that $$\sup_{x \in [a,b]} |\phi_{inv}(x) - (1/x)| = O(\epsilon).$$ In addition, $\phi_{inv}$ satisfies $\|\phi_{inv}\|_{\textrm{path}} = O(b a^{-2})$.
    \end{enumerate}
\end{lemma}
\begin{proof}
    For 1), note that for any $x \in [a,b]$, we have
    \begin{align*}e^x + e^a(x-a+1) &= \int_{a}^{x}(x-t) e^t dt \\
    &= \int_{a}^{b} (x-t)^+ \mu_{exp} dt, \\
    \end{align*}
where $\mu_{exp}dt = e^t dt$ (note that $\|\mu_{exp}\|_{TV} \leq e^b$. We apply Lemma \ref{empproc} with the function $g(x,t) = (x-t)^+$. The Lipschitz constant is bounded at 1 (since ReLU is 1-Lipschitz) and the function values of $g$ over $(a,b)$ are bounded at $2b$. We conclude that there exist $t_1, \dots, t_m \in [a,b]$ such that the function $\phi_{exp}(x) := \frac{\|\mu_{exp}\|_{TV}}{m} \sum_{i=1}^{m} (x-t_i)^+$ satisfies
$$ \sup_{x \in [a,b]} \left| e^x + e^a(x-a+1) - \phi_{exp}(x) \right| \lesssim \frac{2be^b}{\sqrt{m}} \max(b,-a) \sqrt{\log(mb)} = O \left(\frac{e^b}{\sqrt{m}} \right).
$$
If we set $m = O(\epsilon^{-2} e^{2b})$, then the approximation error is $O(\epsilon)$ We note that $\phi_{exp}(x) - (e^a(x-a+1))$ is also a ReLU network, so that (upon renaming $\phi_{exp})$ we have obtained a neural network approximation to $e^x$ on $[a,b]$. Finally, up to an $O(1)$ summand, we have
$$ \|\phi_{exp}\|_{\textrm{path}} = \frac{e^b-e^a}{m} \sum_{i=1}^{m} |t_i| \leq \max(b,-a) (e^b-e^a) = O(e^b).
$$

For 2), first observe that we can approximate the one-dimensional map $x \mapsto x^2$ by a shallow ReLU neural network $\phi_{sq}(x)$ on $[-2R,2R]$ with $O(\epsilon^{-2}R^2)$ neurons. Indeed, for $x \in [-R,R]$, we can write
$$ x^2 = \int_{0}^{x} 2(x-t) dt = \int_{0}^{2R} 2(x-t)^+ dt.
$$
Using Lemma \ref{empproc} (in a similar fashion to part 1) above), we conclude the existence of such an approximating $\phi_{sq}(x) = \frac{4R}{m} \sum_{i=1}^{m} (x-t_i)^{(+)}$. The path seminorm of $\phi_{sq}$ can be bounded by $O(R)$ using the same argument as in part 1). It follows that $xy = \frac{1}{4}((x+y)^2 -(x-y)^2)$ can be approximated by $\phi_{prod}(x,y) := \frac{1}{4}(\phi_{sq}(x+y) - \phi_{sq}(x-y))$ on $[-R,R]^2$. The number of neurons (and path norm constant) of $\phi_{prod}$ is bounded by the number of neurons (and path norm constant) of $\phi_{sq}$, up to a constant multiple.

For 3), the idea is very similar, so we omit some of the details: if $x \in [a,b]$ with $a > 0$, then we have
$$ \frac{1}{x} - \frac{2}{a} + \frac{1}{a^2}{x} = \int_{a}^{b} (x-t)^+ \frac{2}{t^3} dt.
$$
The conclusion then follows from another application of Lemma \ref{empproc}, noting that the total variation of the parameter measure above is $\int_{a}^{b} \frac{2}{t^3} = O(a^{-2})$.
\end{proof}
As a consequence of Lemma \ref{approxbuildingblocks}, we can approximate the map $(x,y) \mapsto \frac{x}{y}$ by a neural network, provided that the domain of the second coordinate is bounded away from 0.
\begin{lemma}\label{approxquot}
    Let $\epsilon > 0$, $R > 0$, and $0 < a < b < \infty$. Let $M = \max(R,\frac{b}{a^2})$. Then there exists a ReLU neural network $\phi_{quot}$ with 2 layers and $O(a^{-4} \epsilon^{-4} R^4 M^2)$ parameters such that
    $$ \sup_{x \in [-R,R], y \in [a,b]} \left| \phi_{quot}(x,y) - \frac{x}{y} \right| = O(\epsilon).
    $$
    Moreover, we have $\| \phi_{quot}\|_{\textrm{path}} = O(M^2 ba^{-2})$.
\end{lemma}
\begin{proof}
    Let $\bar{\epsilon} = (R+1)^{-1} \epsilon$. By Lemma \ref{approxbuildingblocks}, we can find shallow neural networks $\phi$ and $\psi$ satisfying
    $$ \sup_{y \in [a,b]} \left|\phi(y) - \frac{1}{y}\right| \leq \bar{\epsilon}
    $$
    and
    $$ \sup_{(x,y) \in [-M,M]} |\psi(x,y) - xy| \leq \bar{\epsilon}.
    $$
    Let $\phi_{quot}(x,y) = \psi(x,\phi(y))$. Then
    \begin{align*}
        \sup_{x \in [-R,R], y \in [a,b]} \left|\phi_{quot}(x,y) - \frac{x}{y}\right| &\leq \sup_{x \in [-R,R], y \in [a,b]} \left|\frac{x}{y}- x \phi(y)\right| \\ &+ \sup_{x \in [-R,R], y \in [a,b]} |x \phi(y) - \Phi(x,y)|.
    \end{align*}
    For the first term, we have
    \begin{align*}
        \sup_{x \in [-R,R], y \in [a,b]} \left|\frac{x}{y} - x \phi(y)\right| &\leq R \sup_{y \in [a,b]} \left|\frac{1}{y} - \phi(y)\right| \leq R \bar{\epsilon}.
    \end{align*}
    For the second term, note that an inspection of the proof of Lemma \ref{approxbuildingblocks} shows that $\phi$ is $O(a^{-2})$-Lipschitz, so that, up to a constant factor, we have
    $$ \phi([a,b]) \subseteq \left[\phi(0) - \frac{b}{a^2}, \phi(0) + \frac{b}{a^2}\right].
    $$
    This guarantees that
    \begin{align*}
        \sup_{x \in [-R,R], y \in [a,b]} |x \phi(y) - \Phi(x,y)| &:= \sup_{x \in [-R,R], y \in [a,b]} |x \phi(y) - \psi(x,\phi(y))| \\
        &\leq \sup_{x \in [-R,R], y \in [\phi(0) - \frac{b}{a^2}, \phi(0) + \frac{b}{a^2}]} |xz - \psi(x,z)| \\
        &\leq \sup_{(x,y) \in [-M,M]^2} |xz - \psi(xz)| \leq \bar{\epsilon}.
    \end{align*}
    This proves that 
    $$  \sup_{x \in [-R,R], y \in [a,b]} \left|\phi_{quot}(x,y) - \frac{x}{y}\right| \leq (R+1)\bar{\epsilon} = \epsilon.
    $$
    To conclude, we have that $\|\phi_{quot}\|_{\textrm{path}} \leq \|\psi\|_{\textrm{path}} \cdot \|\phi\|_{\textrm{path}} = O(ba^2M^2)$.
\end{proof}
\begin{lemma}\label{approxef}
    Let $f$ satisfy Assumption \ref{approxassump} . Then for any $R > \max(r_f, \sqrt{\frac{1}{\beta}\sup_{\|x\| \leq r_f}|f(x)|})$ and $\epsilon > 0$, there exists a ReLU neural network $\phi_{f,exp}$ with $(L_f+1)$ layers such that $$\sup_{\|x\| \leq R} |\phi_{f,exp}(x) - f(x)| \lesssim e^{\beta R^2} \epsilon.$$ In addition, $\phi_{f,\exp}$ satisfies $\|\phi_{f,exp}\|_{\textrm{path}} = O(e^{\beta R^2} \cdot \eta(R,\epsilon))$.
\end{lemma}
\begin{proof}
By Assumption \ref{approxassump}, there exists an $L_f$-layer ReLU neural network $\phi_f$ with $$\sup_{\|x\| \leq R} |f(x) - \phi_f(x)| \leq \epsilon.$$ By Lemma \ref{approxassump}, there exists a shallow neural network $\phi_{exp}: \R \rightarrow \R$ such that
$$ \sup_{z \in [-\alpha R^2, \beta R^2]} |\phi_{exp}(x) - e^x | \lesssim \epsilon
$$
and $\|\phi_{exp}\|_{\textrm{path}} = O(e^{\beta R^2})$. In turn, the $(L_f+1)$-layer ReLU network $\phi_{f,exp} = \phi_{exp} \circ \phi_f$ satisfies
\begin{align*}
    \sup_{\|x\| \leq R} |\phi_{f,exp}(x) - e^{f(x)}| &\leq \sup_{\|x\| \leq R} |\phi_{f,exp}(x) - (\phi_{exp} \circ f)(x)| + \sup_{\|x\| \leq R} |(\phi_{exp} \circ f)(x) - e^{f(x)}| \\
    &\lesssim e^{\beta R^2} \sup_{\|x\| \leq R} |\phi_{f}(x) - f(x)| + \sup_{z \in [-\alpha \|x\|^2, \beta \|x\|^2]} |\phi_{exp}(z) - e^z| \\
    &\lesssim  e^{\beta R^2} \epsilon.
\end{align*}
In addition, it follows that $$\|\phi_{f,exp}\|_{\textrm{\textrm{path}}} \leq \|\phi_{exp}\|_{\textrm{\textrm{path}}} \cdot \|\phi_{exp}\|_{\textrm{\textrm{path}}} \lesssim e^{\beta R^2} \cdot \eta(R,\epsilon).$$
\end{proof}

\begin{lemma}\label{approxprodef}
    Let $f$ satisfy Assumption \ref{approxassump}. Then for any $R > \max(r_f, \sqrt{\frac{1}{\beta}\sup_{\|x\| \leq r_f}|f(x)|})$ and $\epsilon > 0$, there exists a ReLU neural network $\Phi_f^j(x)$ such that
    $$ \sup_{\|x\| \leq R} \left| \Phi^j_f(x) - x_j e^{f(x)} \right| \lesssim e^{\beta R^2} \epsilon.
    $$
\end{lemma}
In addition, we have $\|\Phi_f^j\|_{\textrm{path}} = O(e^{3 \beta R^2} \cdot \eta(R,\epsilon))$.
\begin{proof}
Let $\phi_{f,exp}$ denote the ReLU network constructed in Lemma \ref{approxef}, so that $\sup_{\|x\| \leq R} |\phi_{f,exp}(x) - e^{f(x)}| \lesssim e^{\beta R^2} \epsilon$ and $\|\phi_{f,exp}\|_{\textrm{path}} \lesssim e^{\beta R^2} \cdot \eta(R,\epsilon)$. This also implies that $\sup_{\|x\| \leq R} |\phi_{f,exp}(x)| \leq Ce^{\beta R^2}$ for a universal constant $C \geq 1$. By Lemma \ref{approxbuildingblocks}, there exists a shallow ReLU neural network $\phi_{prod}: \R^2 \rightarrow \R$ such that
$$ \sup_{|y| \leq R, |z| \leq Ce^{\beta R^2}} |\phi_{prod}(y,z) - yz| \lesssim \epsilon
$$
and $\|\phi_{prod}\|_{path} = O(e^{2 \beta R^2})$. In turn, the $(L_f+2)$-layer ReLU network $\Phi_f^j(x) = \phi_{prod}(x_j,\phi_{f,exp}(x))$ satisfies $\|\Phi_f^j\|_{\textrm{path}} \lesssim e^{2\beta R^2} \cdot \|\phi_{f,exp}\|_{\textrm{path}} \lesssim e^{3 \beta R^2} \cdot \eta(R,\epsilon)$ and
\begin{align*}
    \sup_{\| x \| \leq R} |\Phi_f^j(x) - x_j e^{f(x)}| &\leq \sup_{\|x\| \leq R} |\Phi_f^j(x) - x_j \phi_{f,exp}(x)| + \sup_{\|x\| \leq R} |x_j \phi_{f,exp}(x) - x_j e^{f(x)}| \\
    &\leq \sup_{|y| \leq R, |z| \leq C e^{\beta R^2}} |\phi_{prod}(y,z) - yz| + R \sup_{\|x\| \leq R} |\phi_{f,exp}(x) - e^{f(x)} \| \\
    &\lesssim \epsilon + Re^{\beta R^2} \epsilon \lesssim e^{\beta R^2} \epsilon.
\end{align*}
\end{proof} 

\subsection{Contraction inequality for vector-valued functions}
We present the contraction inequality for vector valued functions, which is a slight modification of Theorem 3 in \citet{maurer2016vector}. The proof of this result can be found in the aforementioned paper.
\begin{prop}\label{vectorcontract}
    Let $\f$ be a separable class of functions from $\R^d$ to $\R^d$, let $\{x_1, \dots, x_N\} \subset B_S$ and let $\Psi: \f \times \R^d \rightarrow \R$ satisfy 
    $$ \Psi(f, x) - \Psi(f',x) \leq L \E_{X_t | X_0 = X_i} \|f(X_t)-f'(X_t)\|, \; \forall f, f' \in \f, \; x \in \R^d.
    $$
    Then it holds that
    $$ \E_{\epsilon_i} \sup_{f \in \f} \sum_{i=1}^{N} \epsilon_i \Psi(f,x_i) \leq \sqrt{2} L \E_{X_t^i|X_0^i = x_i} \E_{\epsilon} \sup_{f \in \f} \sum_{i,k} \epsilon_{ik} f_k(X_t^i),
    $$
    where $\{\epsilon_{ik}\}_{1 \leq i \leq N, 1 \leq k \leq d}$ are independent Rademacher random variables and $f_k$ denotes the $k$-th component of $f$.
\end{prop}
\end{document}